\documentclass[10pt, a4paper, logo, onecolumn]{deepmind}
\usepackage{titletoc}
\renewcommand{\today}{}
\pdftrailerid{redacted}
\usepackage{kantlipsum, lipsum}
\usepackage{dsfont}
\usepackage[authoryear, sort&compress, round]{natbib} 
\usepackage[boxruled,vlined,noend]{algorithm2e}

\usepackage{amsmath}
\usepackage{amssymb}
\usepackage{mathtools}
\usepackage{amsthm}
\usepackage{float}
\usepackage{xcolor}
\usepackage{multirow}
\usepackage{dirtytalk}
\usepackage{nicefrac}
\usepackage{cleveref}
\usepackage[normalsize]{subfigure}

\usepackage{pifont}

\theoremstyle{plain}
\newtheorem{theorem}{Theorem}[section]

\newtheorem{lemma}[theorem]{Lemma}

\theoremstyle{definition}
\newtheorem{definition}[theorem]{Definition}

\theoremstyle{remark}

\title{Unlocking High-Accuracy Differentially Private Image Classification through Scale}

\correspondingauthor{ \{sohamde | lberrada | slsmith | bballe\}@deepmind.com}
\author[*,1]{Soham De}
\author[*,1]{Leonard Berrada}
\author{Jamie Hayes}
\author{Samuel L Smith}
\author{Borja Balle}
\affil[*]{Equal contributions}
 \affil[1]{DeepMind}

\begin{abstract}
Differential Privacy (DP) provides a formal privacy guarantee preventing adversaries with access to a machine learning model from extracting information about individual training points. 
Differentially Private Stochastic Gradient Descent (DP-SGD), the most popular DP training method for deep learning, realizes this protection by injecting noise during training.
However previous works have found that DP-SGD often leads to a significant degradation in performance on standard image classification benchmarks.
Furthermore, some authors have postulated that DP-SGD inherently performs poorly on large models, since the norm of the noise required to preserve privacy is proportional to the model dimension. 
In contrast, we demonstrate that DP-SGD on over-parameterized models can perform significantly better than previously thought.
Combining careful hyper-parameter tuning with simple techniques to ensure signal propagation and improve the convergence rate, we obtain a new SOTA without extra data on CIFAR-10 of 81.4\% under $\mathbf{(8, 10^{-5})}$-DP using a 40-layer Wide-ResNet, improving over the previous SOTA of 71.7\%.
When fine-tuning a pre-trained NFNet-F3, we achieve a remarkable 83.8\% top-1 accuracy on ImageNet under $\mathbf{(0.5, 8\cdot 10^{-7})}$-DP. Additionally, we achieve 86.7$\%$ top-1 accuracy under $\mathbf{(8, 8 \cdot 10^{-7})}$-DP, only 4.3$\%$ below the current non-private SOTA for this task.
We believe our results are a significant step towards closing the accuracy gap between private and non-private image classification.
\end{abstract}

\begin{document}
\maketitle

\section{Introduction}

Machine learning models trained with standard pipelines can be attacked by an adversary that seeks to reveal the data on which the model was trained.
For example, \citet{carlini2021extracting} showed that adversaries can generate and detect text sequences from the training set of a large transformer language model, while \citet{balle2022reconstructing} showed that powerful adversaries can reconstruct images in the training set of a classifier trained on CIFAR-10.
Alongside other results \citep{DBLP:journals/corr/abs-2112-03570,DBLP:conf/icml/Choquette-ChooT21,DBLP:journals/corr/abs-2102-02551}, these studies demonstrate that models trained on sensitive datasets present a significant privacy risk.

Differential Privacy (DP) \citep{dwork2006calibrating} is the gold standard technique for mitigating privacy attacks aimed at leaking individual training examples, and it has already been adopted in practice by a range of public and private organizations \citep{erlingsson_2014,differential_privacy_team_2017,DBLP:conf/kdd/Abowd18,nayak_2020,bird_2020,mcmahan_thakurta_2022}.
A differentially private algorithm is a randomized algorithm providing a formal guarantee that any single example in the training set can only influence the output distribution of the algorithm by a small, pre-specified amount.
This privacy guarantee, denoted $(\varepsilon,\delta)$-DP, is defined by two parameters $(\varepsilon, \delta)$, which we refer to as the \emph{privacy budget}.
The smaller these two parameters are, the closer the output distributions between training sets that differ by a single example, and therefore the more difficult it is for an adversary to infer whether any single data point was included during training.

The most popular method for training neural networks with DP is Differentially Private Stochastic Gradient Descent (DP-SGD) \citep{abadi2016deep}. 
DP-SGD replaces the usual mini-batch gradient estimate of SGD with a privatized version, in which the gradient of each training example is clipped to a maximum norm. 
In addition, Gaussian noise proportional to the clipping norm is added to the sum of the clipped gradients, which is sufficient to mask the contribution of any single example to the sum.
Each evaluation of a (privatized) mini-batch gradient incurs a privacy cost, and a privacy accountant \citep{abadi2016deep, mironov2019r} is used to track the total privacy budget spent throughout training. 
These parameters increase with every mini-batch seen during training, and decrease with the scale of the noise added, thus limiting the number of training iterations we can perform at a fixed privacy budget while keeping the variance in the gradient estimate under control.

Training with DP-SGD involves a delicate balancing act between different hyper-parameters such as the amount of added noise, the batch size, and the number of training iterations, in order to reach optimal performance within a specified privacy budget.
In particular, the noise added to the gradient is a significant barrier to optimization, typically resulting in a significant degradation in performance compared to standard non-private training \citep{klause2022differentially, DormannFAP21, Kurakin22}. 
Furthermore, several authors have postulated that highly over-parameterized models, which perform well in non-private settings, do not work well when used with DP-SGD, because the norm of the added noise increases with the dimension of the gradient \citep{DBLP:conf/iclr/TramerB21, shen2021towards, YuZ0L21, Kurakin22}, leading to a ``curse of dimensionality''. 
Consequently, many works have focused either on developing specialized architectures for private training \citep{DBLP:conf/aaai/PapernotT0CE21, DBLP:conf/iclr/TramerB21}, or on reducing the dimensionality of the model during training \citep{YuZ0L21, ZhouW021, DBLP:journals/corr/abs-2103-01294,DBLP:journals/corr/abs-2203-11481}.
\begin{figure}[t]
    \centering
    \subfigure[CIFAR-10 without extra data]{
        \includegraphics{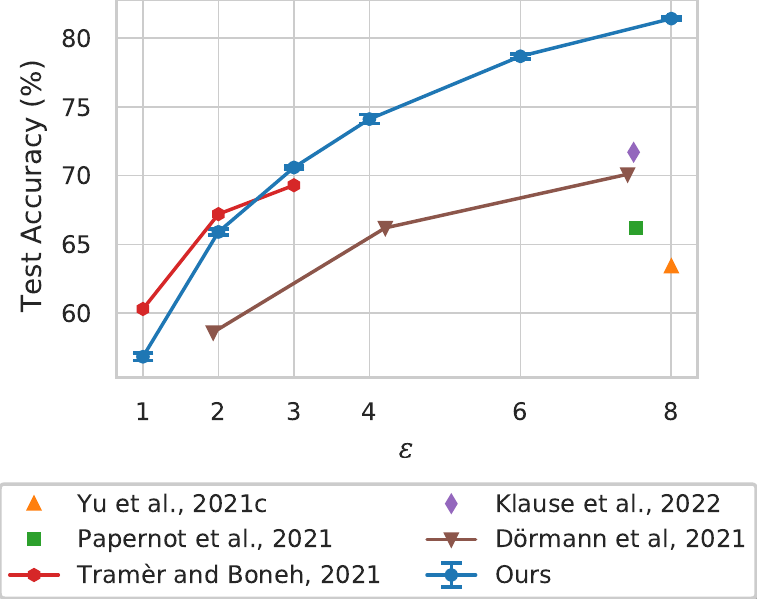}  
        \label{fig:headline_a}
    } 
    \hspace{3mm}
    \subfigure[ImageNet with extra data]{
        \raisebox{10pt}{\includegraphics{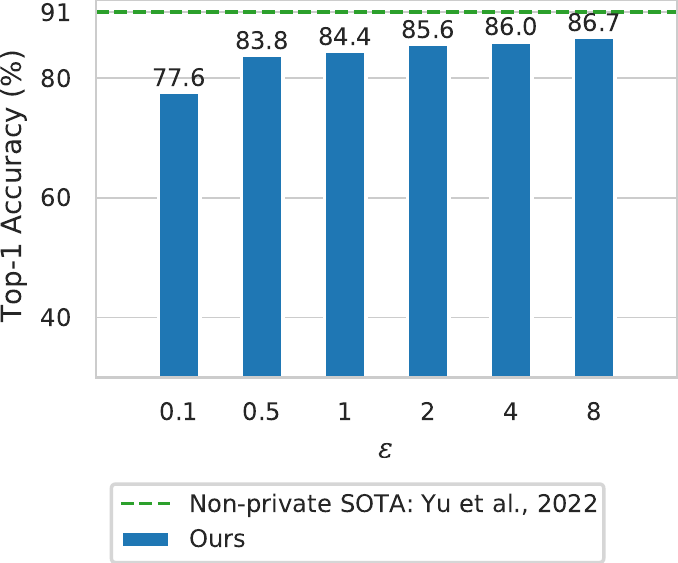}}
        \label{fig:headline_b}
    }
    \caption{
    (a) When training on CIFAR-10 without additional data, we improve on previously published results under $(\varepsilon, 10^{-5})$-DP whenever $\varepsilon \geq 3$.
    At $\varepsilon=8$, we improve on the previous SOTA of \citet{klause2022differentially} by $9.7\%$. 
    Note we report the mean and standard error across 5 independent runs. 
    (b) When fine-tuning a pre-trained NFNet-F3 \citep{DBLP:conf/icml/BrockDSS21} on ImageNet under $(8, 8 \cdot 10^{-7})$-DP, we achieve $86.7\%$ top-1 accuracy, only $4.3\%$ below the current non-private SOTA of $91.0\%$ \citep{yu2022coca}. We also obtain $83.8\%$ top-1 accuracy under a much tighter $(0.5, 8\cdot 10^{-7})$-DP guarantee, which exceeds the performance of many popular non-private models (e.g., ResNet-50). 
}
    \label{fig:headline_fig1}
\end{figure} 

On the contrary, we show that standard over-parameterized architectures, which achieve close to state-of-the-art performance in non-private training, can also perform very well when trained using DP-SGD if properly tuned. 
To achieve this, we introduce a number of techniques which help convergence and ensure trainability at initialization, and we explore the benefits of using pre-trained models. 
Our main contributions are listed below:
\begin{itemize}
    \item We describe a set of simple techniques which, when combined, significantly improve the performance of DP-SGD. 
    First, we revisit ideas that have previously been identified as useful for private training, including using large batch sizes \citep{DBLP:conf/iclr/McMahanRT018}, and replacing batch normalization layers with alternatives that ensure good signal propagation at initialization \citep{DBLP:journals/corr/abs-2007-05089}.
    In addition, we propose further modifications that improve the convergence rate of DP-SGD, and that have not previously been used for private training.
    Specifically, we suggest using weight standardization in convolutional layers \citep{qiao2019micro}, leveraging the benefits of data augmentation by averaging per-example gradients across multiple augmentations of the same image before the clipping operation \citep{hoffer19} and applying parameter averaging techniques \citep{polyak1992acceleration}.
    \vspace{3mm}
    \item Applying the techniques above, we significantly improve the performance of DP-SGD when training randomly initialized over-parameterized models. 
    Training Wide-ResNets \citep{DBLP:conf/bmvc/ZagoruykoK16} on CIFAR-10 without extra data, we achieve a new SOTA of 81.4\% under $(8, 10^{-5})$-DP.
    This is a substantial improvement on the previous SOTA of 71.7\% achieved under $(7.5, 10^{-5})$-DP \citep{klause2022differentially}. 
    As shown in \Cref{fig:headline_a}, we achieve SOTA results on this task across a range of $\varepsilon$ values between 3 and 8. 
    We also achieve a new SOTA top-1 accuracy on ImageNet of 32.4$\%$ under $(8, 8\cdot 10^{-7})$-DP when training a Normalizer-Free ResNet-50 (NF-ResNet-50) \citep{he2016deep, brock2020characterizing}.
    \vspace{3mm}
    \item We show that non-private pre-training on public/non-sensitive data, followed by fine-tuning with DP-SGD on the private dataset, yields remarkable performance benefits on image classification benchmarks. For example, when privately fine-tuning an NF-ResNet-200 pre-trained on JFT-300M \citep{sun2017revisiting}, we achieve 81.3$\%$ top-1 accuracy on ImageNet under $(8, 8\cdot 10^{-7})$-DP. We observe further improvements in performance from increasing both the size of the model and the size of the pre-training dataset, achieving 86.7$\%$ top-1 accuracy under $(8, 8\cdot 10^{-7})$-DP with an NFNet-F3 \citep{DBLP:conf/icml/BrockDSS21} pre-trained on JFT-4B. This network also obtains 83.8$\%$ top-1 accuracy under a much tighter privacy budget of $(0.5, 8 \cdot 10^{-7})$-DP. For comparison, fine-tuning the same pre-trained network on ImageNet without privacy reaches 88.5$\%$. 
\vspace{3mm}
    \item We provide novel insights into how optimal hyper-parameters relate to each other when training with DP. 
    We empirically observe that (1) there is an optimal budget of training iterations given a fixed batch-size, (2) larger batch sizes improve validation accuracy but require more training epochs after the batch size exceeds a certain threshold, and (3) the optimal choice of learning-rate for DP-SGD is proportional to the batch size when the batch size is small but constant for large batch sizes, similar to non-private training.
\end{itemize}

We summarize our key results in \Cref{tab:summary_results}, with our SOTA results shown in bold.\footnote{Shortly after our paper was released, \citet{klause2022differentially} updated their paper to combine our techniques with their suggested ScaleNorm,
achieving a new SOTA test accuracy of 82.5\% on CIFAR-10 without extra data at $\varepsilon=8$. Additionally, \citet{bu2022scalable} also published their work showing the benefits of large models when fine-tuning with differential privacy by achieving a new SOTA test accuracy with extra data of 96.7\% on CIFAR-10 and 83.0\% on CIFAR-100 under $(1, 10^{-5})$-DP using a large ViT model pre-trained on ImageNet.}
We emphasize that all our results use standard vision architectures which have been shown to work well for non-private training.
We believe these results are a significant step towards practically useful differentially private image classification.

\begin{table}[t!]
    \centering
    \caption{
    A summary of the best results provided in this paper when training with DP-SGD. 
    All numbers in bold are SOTA.
    For the CIFAR-10 and CIFAR-100 experiments, we report the median accuracy across 5 independent runs. 
    All experiments on CIFAR use Wide-ResNets with group normalization, while the ImageNet and Places-365 experiments use NF-ResNets or NFNets. 
    See relevant sections for further details.}
    \label{tab:summary_results}
    \begin{tabular}{lccccccc}
    \toprule
    \multirow{2}{*}{Dataset} & \multirow{2}{*}{Pre-Training} & \multicolumn{4}{c}{Top-1 Accuracy (\%)} & \multirow{2}{*}{$\delta$} & \multirow{2}{*}{Section} \\
    \cmidrule(lr){3-6}
    & Data &			$\varepsilon=1$	& $\varepsilon=2$	& $\varepsilon=4$	& $\varepsilon=8$	& & \\
    \midrule
    CIFAR-10 & -- & 56.8 & 65.9 & \textbf{73.5} & \textbf{81.4} & $10^{-5}$ & \ref{sec:cifar_sota} \\
    ImageNet & -- & -- & -- & -- & \textbf{32.4} & $8 \cdot 10^{-7}$ & \ref{sec:imagenet_scratch} \\
    \midrule
    CIFAR-10 & ImageNet & \textbf{94.7} & \textbf{95.4} & \textbf{96.1} & \textbf{96.7} & $10^{-5}$ & \ref{sec:cifar_finetune} \\
    CIFAR-100 & ImageNet & \textbf{70.3} & \textbf{74.7} & \textbf{79.2} & \textbf{81.8} & $10^{-5}$ & \ref{sec:cifar_finetune} \\
    ImageNet & JFT-4B & \textbf{84.4} & \textbf{85.6} & \textbf{86.0} & \textbf{86.7} & $8 \cdot 10^{-7}$ & \ref{sec:imagenet} \\
    Places-365 & JFT-300M & -- & -- & -- & \textbf{55.1} & $5 \cdot 10^{-7}$ & \ref{sec:places} \\
    \bottomrule
    \end{tabular}
\end{table}

\paragraph{Paper outline.} 
We provide a brief introduction to Differential Privacy and DP-SGD in \Cref{sec:background}, where we also discuss the challenges that arise when applying DP-SGD to deep networks. 
In \Cref{sec:dptraining_protocol}, we describe a range of techniques that enhance the performance of networks trained with DP-SGD, achieving SOTA performance on CIFAR-10 and ImageNet when training without additional data. 
In \Cref{sec:pre-training}, we show that privately fine-tuning strong pre-trained models dramatically improves the performance of private image classification. 
Finally, we provide additional insights into how the hyper-parameters of DP-SGD influence performance in \Cref{sec:cifar_tuning}.

\paragraph{Reproducibility.} 
To help researchers reproduce and verify our results, we are releasing  the implementation of DP-SGD used in our experiments at \url{https://github.com/deepmind/jax_privacy}.
We also provide the configuration scripts and pre-trained checkpoints necessary to reproduce all of our results on CIFAR-10 and CIFAR-100, as well as our results on ImageNet without extra data.
We provide further details about our DP-SGD implementation in \Cref{app:jax-dp-sgd}, together with a description of the steps we undertook to audit its correctness.

\section{Background}
\label{sec:background}
\subsection{Differential Privacy (DP)}
Differential privacy (DP) is a formal privacy guarantee that applies to randomized data analysis algorithms. 
By construction, differentially private algorithms prevent an adversary that observes the output of a computation from inferring any property pertaining to individual data points in the input data used during the computation. 
The strength of this guarantee is controlled by two parameters: $\varepsilon > 0$ and $\delta \in [0,1]$. 
Roughly speaking, $\varepsilon$ bounds the log-likelihood ratio of any particular output that can be obtained when running the algorithm on two datasets differing in a single data point, and $\delta$ is a small probability which bounds the occurrence of infrequent outputs that violate this bound. 
The privacy guarantee becomes stronger as both parameters get smaller. 
A standard rule of thumb states that, to obtain meaningful privacy, $\varepsilon$ should be a small constant while $\delta$ should be smaller than $1/ N$, where $N$ is the size of the input dataset. 
More formally, we have the following.

\begin{definition}[Differential Privacy \citep{dwork2006calibrating}]
Let $A: \mathcal{D} \to \mathcal{S}$ be a randomized algorithm, and let $\varepsilon > 0$, $\delta \in [0, 1]$.
We say that $A$ is $(\varepsilon, \delta)$-DP if for any two neighboring datasets $D, D' \in \mathcal{D}$ differing by a single element, we have that
\begin{align}
\forall \: S \subset \mathcal{S}, \: \mathbb{P}[A(D) \in S] \leq \exp(\varepsilon) \mathbb{P}[A(D') \in S] + \delta \enspace.
\label{eq:dp_definition}
\end{align}
\end{definition}

The privacy protection afforded by DP holds under an exceedingly strong threat model: inferences about individuals are protected even in the face of an adversary that has full knowledge of the DP algorithm, unbounded computational power, and arbitrary side knowledge about the input data.
Furthermore, DP satisfies a number of appealing properties from the algorithm design standpoint, including preservation under post-processing and a smooth degradation with multiple accesses to the same data.
These properties are exploited in the construction of complex DP algorithms based on the combination of small building blocks that inject carefully calibrated noise into operations that access the data.
The magnitude of the noise required to satisfy the privacy guarantee increases with the strength of the privacy parameters, leading to an unavoidable trade-off between utility and privacy, as illustrated by the Fundamental Law of Information Recovery \citep{TCS-042}.

Together, the strength of the formal guarantee it provides and the variety of tools available for the construction of DP algorithms, have led to the growing adoption of DP as a gold standard for privacy-preserving machine learning.
For convex learning problems there exists a variety of methods for obtaining differentially private algorithms, including output perturbation \citep{DBLP:journals/jmlr/ChaudhuriMS11,DBLP:conf/sigmod/0001LKCJN17}, objective perturbation \citep{DBLP:journals/jmlr/ChaudhuriMS11,DBLP:journals/jmlr/KiferST12} and gradient perturbation \citep{DBLP:conf/globalsip/SongCS13,DBLP:conf/focs/BassilyST14}. 
The nature of convex problems enables the formal analysis of the privacy-utility offered by these algorithms, and by now there are large classes of problems for which algorithms achieving (nearly) optimal privacy-utility trade-offs are known \citep{DBLP:conf/focs/BassilyST14,DBLP:conf/nips/TalwarTZ15,DBLP:conf/stoc/FeldmanKT20,pmlr-v130-song21a,DBLP:conf/icml/AsiFKT21}.
For non-convex learning problems, the range of available algorithms is more limited and privacy-utility trade-offs are harder to analyze theoretically.
Nonetheless, for such problems there exist two families of algorithms that have been shown to achieve reasonable privacy-utility-computation trade-offs in practice: gradient perturbation applied to standard optimizers like SGD \citep{abadi2016deep}, and private aggregation of teacher ensembles \citep{DBLP:conf/iclr/PapernotSMRTE18}.
In this work we focus on the former, which is most commonly used.

\subsection{Differentially Private Stochastic Gradient Descent (DP-SGD)}
\label{sec:background_dpsgd}

In this work, we assume that the differentially private algorithm $A$ (see \Cref{eq:dp_definition}) is a learning algorithm that maps a training dataset $D=\{(x_i, y_i)\}_{1 \leq i \leq N}$ to a vector of learned neural network parameters $w \in \mathcal{S} = \mathbb{R}^p$.
Let $\mathcal{L}(w, x, y)$ denote the learning objective (e.g., the cross-entropy loss), given the model parameters $w$, input example $x$ and label $y$. 
For convenience, we use the shorthand notation $l_i(w) = \mathcal{L}(w, x_i, y_i)$.

In the non-private setting, a parameter update using Stochastic Gradient Descent (SGD) at iteration $t$ draws $B$ examples at random from the dataset, and performs an update of the form:
$$
w^{(t+1)} = w^{(t)} - \eta_t \frac{1}{B} \sum\limits_{i \in \mathcal{B}_t} \nabla l_i (w^{(t)}) \enspace,
$$
where $\eta_t$ is the step-size for the $t^{th}$ update, $\nabla$ denotes the gradient operator, and $\mathcal{B}_t$ represents the set of examples sampled at iteration $t$ with $|\mathcal{B}_t|=B$. 
In order to make this algorithm differentially private, we apply the following modifications. 
First, the gradient for each example in the mini-batch is clipped to a maximal norm $C$, and second, Gaussian noise with standard deviation proportional to $C$ is added to the mean of the clipped gradients.
Let $\texttt{clip}_C: v \in \mathbb{R}^p \mapsto \min\left\{1,  \tfrac{C}{\|v\|_2}\right\} \cdot v \in \mathbb{R}^p$ denote the clipping function which re-scales its input so that the output has a maximal $\ell_2$ norm of $C$.
The new update step is:
\begin{align}
    w^{(t+1)} = w^{(t)} - \eta_t \left\{ \frac{1}{B} \sum\limits_{i \in \mathcal{B}_t} \texttt{clip}_C \left(\nabla l_i (w^{(t)}) \right) + \frac{\sigma C}{B} \xi \right\} \enspace, \label{eq:dpsgd}
\end{align}
where $\xi \sim \mathcal{N}(0, I_p)$ is a standard $p$-dimensional Gaussian random variable and $\sigma$ specifies the standard deviation of the added noise.
The resulting algorithm is called Differentially Private-Stochastic Gradient Descent (DP-SGD) \citep{abadi2016deep}.
Intuitively, performing a model update using \Cref{eq:dpsgd} provides differential privacy because adding Gaussian noise with standard deviation proportional to $C$ is sufficient to mask the contribution of any single example whose clipped gradient has norm less than or equal to $C$.
While we use privatized SGD as our optimizer throughout this work, a similar privatization method can also be used in combination with other first-order optimization algorithms, such as SGD with momentum or Adam \citep{https://doi.org/10.48550/arxiv.1812.06210}.

Throughout this paper we use a modified version of DP-SGD where the privatized gradient is normalized by $C$:
\begin{align}
    w^{(t+1)} = w^{(t)} - \eta_t \left\{ \frac{1}{B} \sum\limits_{i \in \mathcal{B}_t}  \frac{1}{C}\texttt{clip}_C \left(\nabla l_i (w^{(t)}) \right) + \frac{\sigma}{B} \xi \right\} \enspace. \label{eq:dpsgd_rescaled}
\end{align}
This is a re-parameterization of \Cref{eq:dpsgd} in which the learning rate $\eta_t$ absorbs a factor of $C$. 
This has no effect on the privacy guarantees, but ensures the clipping norm does not influence the scale of the update, which simplifies hyper-parameter tuning (See \Cref{app:clipping_rescaling}).  
Note that to preserve the DP guarantees, we must divide by $C$ \emph{after} the clipping operation.
\Cref{app:implementation} provides further details about our DP-SGD implementation, including a description of our approach to virtual batching to enable training with large batch sizes.

\paragraph{Privacy accounting.}
The privacy guarantee of DP-SGD is determined by three parameters: the standard deviation $\sigma$, the sampling ratio $q = B / N$ and the number of training iterations $T$.
In practice, the privacy budget $(\varepsilon, \delta)$ is usually fixed, and these three hyper-parameters are chosen to provide the best possible performance within this budget. 
There may also be additional practical constraints (e.g., the maximum compute budget available).
The privacy calibration process is performed using a privacy accountant: a numerical algorithm providing tight upper bounds for the privacy budget as a function of the hyper-parameters \citep{abadi2016deep}, which in turn can be combined with numerical optimization routines to optimize one hyper-parameter given the privacy budget and the other two hyper-parameters.
In this work we use the accounting method for DP-SGD proposed by \cite{mironov2019r} and implemented in TensorFlow Privacy \citep{tensorflowprivacy}.
This privacy accountant relies on a \say{composition} analysis across iterations, which allows us to release not only the final model, but also \emph{every} intermediate model obtained during training (under the same privacy budget).

\subsection{Challenges of DP-SGD}
\label{sec:background:challenges}

As described above, there are three key differences between DP-SGD and non-private SGD:
(1) the per-example gradients are clipped to a maximal $\ell_2$ norm before they are averaged, (2) Gaussian noise is added to the average of the clipped gradients, and (3) the maximum number of updates allowed within the privacy budget is bounded, and depends on the batch size/added noise.
These differences introduce a number of challenges:

\paragraph{Hyper-parameter tuning and regularization.}
The noise added to the gradient estimate in the DP-SGD update (\Cref{eq:dpsgd_rescaled}) is a significant barrier to efficient optimization, and if we reduce the scale of this noise, the number of training iterations allowed within the privacy budget decreases. 
This constraint alters the optimal values of key hyper-parameters like the batch size/learning rate, and defaults from non-private training can be highly sub-optimal \citep{DBLP:conf/aaai/PapernotT0CE21}.
Consequently, DP-SGD requires careful hyper-parameter tuning.\footnote{In this paper we do not account for the privacy cost of hyper-parameter tuning \citep{papernot2021hyperparameter}, and instead favour thorough sweeps on the same dataset in order to better understand how the hyper-parameters of DP-SGD influence model performance.}
 
We also found in our experiments that, when training with DP-SGD, improvements in training accuracy usually translate directly to improved generalization, without requiring strong regularization. 
Inspired by this observation, our philosophy is that methods that reduce the number of training iterations required to reach high training accuracy in non-private training are likely to improve the test accuracy achieved in private training. 
Consistent with this approach, it is usually beneficial to remove explicit regularization methods.

\paragraph{Bias and variance of the DP-SGD update.} 
The gradient estimator used by DP-SGD is biased because of the use of per-example gradient clipping, and in general it does not correspond to the gradient of any differentiable function \citep{pmlr-v130-song21a}.
More importantly, the clipping norm $C$ introduces a bias-variance trade-off \citep{ChenWH20, DBLP:journals/corr/abs-1905-03871}. 
This can be viewed from the DP-SGD update shown in \Cref{eq:dpsgd_rescaled}.
When $C$ is very large, $\texttt{clip}_C$ is the identity function, so the privatized gradient is an unbiased estimator of the true gradient, but the clipped gradient $\frac{1}{C}\texttt{clip}_C \left(\nabla l_i (w^{(t)}) \right)$ is very small compared to the noise (which is independent of $C$) -- overall the privatized gradient estimate has low bias and high variance. 
Conversely, if $C$ is small, the clipping operation introduces bias, but the clipped gradient $\frac{1}{C}\texttt{clip}_C \left(\nabla l_i (w^{(t)}) \right)$ is larger, and thus is not necessarily small compared to the noise -- overall the privatized gradient estimate has high bias and low variance.
Note that when $C$ is very small (smaller than the smallest per-example gradient norm), further lowering $C$ does not change $\frac{1}{C}\texttt{clip}_C \left(\nabla l_i (w^{(t)}) \right)$, which indicates that the bias and variance in the update both approach a constant as $C \rightarrow 0$. 
Intriguingly, previous work has observed that wide ranges of the clipping norm $C$ can provide near optimal performance provided that (i) the clipping norm is small enough, and (ii) the learning-rate is re-scaled accordingly \citep{Li2021,Kurakin22}.
This suggests that reducing the variance introduced by noise may be more important than reducing the bias introduced by clipping.

\paragraph{Making standard models work.}
Differentially private training has recently obtained promising results with standard architectures in NLP, both when training a BERT model \citep{devlin2018bert} from random initialization \citep{anil2021large}, and when fine-tuning a large Transformer language model \citep{vaswani2017attention} from a pre-trained set of parameters \citep{Li2021,Yu2021lm}.
However, similar results have not been obtained in computer vision, and the literature does not provide clear recommendations on which model architectures perform well. For instance, surveying recent research on private training for CIFAR-10, \citet{Kurakin22}, \citet{DBLP:conf/aaai/PapernotT0CE21} and \citet{DormannFAP21}
use variants of shallow VGG models \citep{DBLP:journals/corr/SimonyanZ14a}, while \citet{DBLP:conf/iclr/TramerB21} use ScatterNets \citep{DBLP:conf/cvpr/OyallonM15} to train linear models on handcrafted features, achieving an impressive $69.3\%$ test accuracy under a tight privacy budget of $(3, 10^{-5})$-DP. 
Finally, \citet{klause2022differentially} achieve the SOTA test accuracy for $\varepsilon \leq 8$ without extra data of $71.7\%$ when training a shallow 9-layer residual network \citep{he2016deep} under $(7.5, 10^{-5})$-DP.

The $\ell_2$ norm of the noise added in the DP-SGD update scales proportional to the dimension of the gradient (the number of parameters). 
This observation has led many researchers to believe that standard over-parameterized models will perform poorly with DP-SGD, and instead focus on reducing the explicit or implicit dimension of the update, either through the use of small models/hand-crafted features \citep{DBLP:conf/iclr/TramerB21} or through dimensionality reduction techniques \citep{yu2021large, YuZ0L21}. 
Another key obstacle to the use of standard models for private training in computer vision has been that in order to provide tight DP guarantees, DP-SGD requires that the gradients evaluated on different training examples are independent. 
This excludes the use of any method that enables communication between training examples, such as batch normalization \citep{DBLP:conf/icml/IoffeS15}, which until recently has been almost ubiquitous in standard vision architectures \citep{he2016deep, DBLP:conf/bmvc/ZagoruykoK16,tan2019efficientnet, DBLP:conf/icml/BrockDSS21, dosovitskiy2020image}.

\section{Improving the Privacy-Utility Trade-off of DP-SGD in Image Classification}
\label{sec:dptraining_protocol}

In this section, we describe the key techniques we use to enhance the performance of models trained with DP-SGD on standard image classification benchmarks. 
In all the experiments in this section we train randomly initialized models without using any extra data. 
In \Cref{sec:convergence} we provide an empirical ablation of these techniques on CIFAR-10, while in \Cref{sec:cifar_sota,sec:imagenet_scratch} we evaluate the performance of our best models on CIFAR-10 and ImageNet. 
We focus on standard architectures popular in the computer vision community, since we believe that these models represent the most promising avenue for achieving long-term progress.

\subsection{Training on CIFAR-10 Without Additional Data -- An Ablation Study}
\label{sec:convergence}

In \Cref{table:method_ablation} we provide an ablation study of a range of techniques, which collectively significantly enhance the performance of DP-SGD when training from random initialization on CIFAR-10 \citep{krizhevsky2009learning}. 
These techniques include replacing batch normalization with group normalization \citep{DBLP:journals/ijcv/WuH20}, using large batch sizes, weight standardization \citep{qiao2019micro}, a modification to DP-SGD which we call augmentation multiplicity \citep{hoffer19, touvron2021training, Fort2021} and parameter averaging \citep{polyak1992acceleration}. 
Prior research has already identified group normalization and large batch sizes as beneficial for training with DP-SGD \citep{DBLP:journals/corr/abs-2007-05089, LuoW0021,Kurakin22,yu2021large,Li2021,anil2021large,DormannFAP21}, while to our knowledge the other techniques have not previously been used for private training. 
We first define our baseline model/training pipeline, and then discuss each modification in turn. 
For all experiments in this section, we split the official CIFAR-10 training dataset of 50K examples into a training set of 45K examples and a validation set of 5K examples. 
We train with DP-SGD on this reduced training set under $(8, 10^{-5})$-DP. 
In order to avoid tuning on the official held-out test set, throughout this subsection we report accuracies obtained on the training and validation set only.

\paragraph{Baseline model and private training pipeline for CIFAR-10.}
As our baseline for CIFAR-10, we study the Wide-ResNet (WRN) model family \citep{DBLP:conf/bmvc/ZagoruykoK16}.
We note that WRN models obtain high accuracies on CIFAR-10 when trained non-privately: the 16 layer network with width factor 4 (denoted as WRN-16-4) achieves >94$\%$ test accuracy, while the WRN-40-4 achieves >95$\%$ \citep{DBLP:conf/bmvc/ZagoruykoK16}. 
The model parameters are initialized using Gaussian random variables following \citet{glorot2010understanding}. 
We train using DP-SGD without Momentum (\Cref{eq:dpsgd_rescaled}). 
We observed no benefit from decaying the learning rate during training, and therefore use a constant learning rate in all experiments reported in this paper. 
We also do not use weight decay or dropout which we found to reduce both training and validation accuracies for private training. 
This is similar to observations made in \citet{DBLP:conf/iclr/TramerB21}, although \citet{anil2021large} reported improved performance when using weight decay with DP-Adam on BERT models. 
Unless otherwise specified, we train without data augmentation.
For all experiments in this subsection, we tune the learning rate $\eta$ and the noise parameter $\sigma$ on the validation set. 
We fix the clipping norm $C=1$ for all experiments in this paper. Given a target privacy guarantee $(\varepsilon, \delta)$-DP and specific hyper-parameter settings for $\sigma$ and $q=B/N$, we compute the maximum number of training iterations $T$ using the privacy accountant discussed in \Cref{sec:background_dpsgd}. 
We provide additional details for each of our experiments in \Cref{app:experimental_details}.

\begin{table}[t]
\caption{\label{table:method_ablation}
    An ablation study on the effect of a range of architectural modifications and changes to the training pipeline for models trained on CIFAR-10 under $(8, 10^{-5})$-DP. 
    We report median and standard deviation values over 5 runs. 
    The baseline is a WRN-40-4 without batch normalization trained with DP-SGD using batch-size 256 without data augmentation. 
    We report accuracy on our own validation set, not the official test set.
}
\begin{center}
\begin{tabular}{lcc}
\toprule [0.15em]
 & \multicolumn{2}{c}{Accuracy (\%)} \\
 \cmidrule{2-3}
 & Validation & Training \\
\midrule [0.1em]
Baseline (WRN-40-4 w/o batch normalization) & 50.8 \: {\color{gray} (0.7)} & 51.2 \: {\color{gray} (0.7)}  \\
+ Group normalization (16 groups) & 66.3 \: {\color{gray} (0.6)} & 67.9 \: {\color{gray} (0.3)}  \\
+ Larger batch size (batch size of 4096) & 70.0 \: {\color{gray} (0.6)}  & 73.4 \: {\color{gray} (0.9)} \\
+ Weight standardization & 71.2 \: {\color{gray} (1.0)} & 74.7 \: {\color{gray}  (1.3)} \\
+ Augmentation multiplicity (16 augmentations) & 78.4 \: {\color{gray} (0.9)}  & 79.4  \: {\color{gray} (0.9)} \\
+ Parameter averaging (exponential moving average) & 79.7 \: {\color{gray} (0.2)}  & 81.5 \: {\color{gray} (0.2)} \\
\bottomrule[0.15em]
\end{tabular}
\end{center}
\end{table}

\paragraph{Training deep networks without batch normalization.} 
As mentioned previously, most computer vision architectures, including the WRN model family, contain batch normalization layers, which are not compatible with DP-SGD. 
Therefore, for our experiments on WRNs, we replace all batch normalization layers by group normalization layers \citep{DBLP:journals/ijcv/WuH20}, as has been previously done by several authors \citep{DBLP:journals/corr/abs-2007-05089,LuoW0021,Kurakin22,yu2021large}. 
Group normalization splits the channels of the hidden activations of a single image into groups and normalizes the hidden activations within each group. 
It therefore does not break the independence between gradients evaluated on different examples.
Note however that it is crucial to place the group normalization layers on the residual branch of the network to recover the benefits of batch normalization for training deep networks \citep{DBLP:conf/nips/DeS20}.
We fix the number of groups of the group normalization layers to 16 in all our experiments.
In \Cref{table:method_ablation}, we show that this simple change significantly improves the performance of deep WRNs when training with DP-SGD.

A range of other alternatives to batch normalization have recently been identified in non-private deep learning \citep{zhang2018fixup, DBLP:conf/nips/DeS20, brock2020characterizing, kolesnikov2020big}. 
These alternatives can train very deep networks \citep{zhang2018fixup, DBLP:conf/nips/DeS20} and some even achieve superior train and test accuracies on standard benchmarks such as ImageNet classification \citep{brock2020characterizing, DBLP:conf/icml/BrockDSS21}. 
We will provide additional results when training Normalizer-Free ResNets \citep{brock2020characterizing} in later sections.

\begin{figure}[t]
\begin{minipage}[b]{0.45\linewidth}
     \centering
    \includegraphics{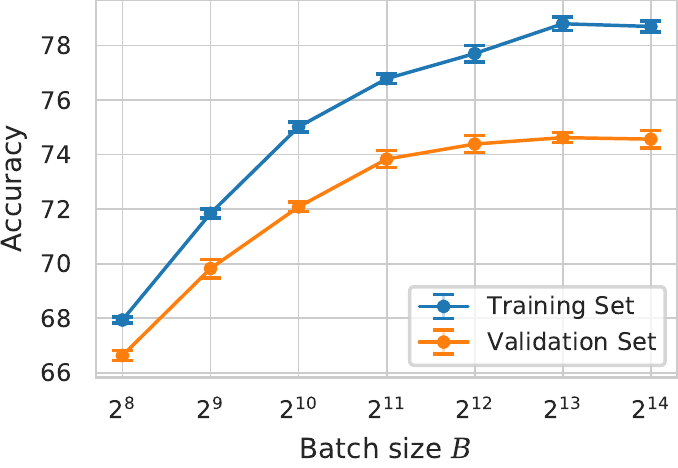}
    \caption{
    Increasing batch sizes on the WRN-16-4 model leads to improved training and validation accuracy under $(8,10^{-5})$-DP. 
    We plot the mean and standard error across 5 independent runs.}
    \label{figure:batchsize_ablation}
\end{minipage}
\hspace{0.5cm}
\begin{minipage}[b]{0.45\linewidth}
    \centering
    \includegraphics{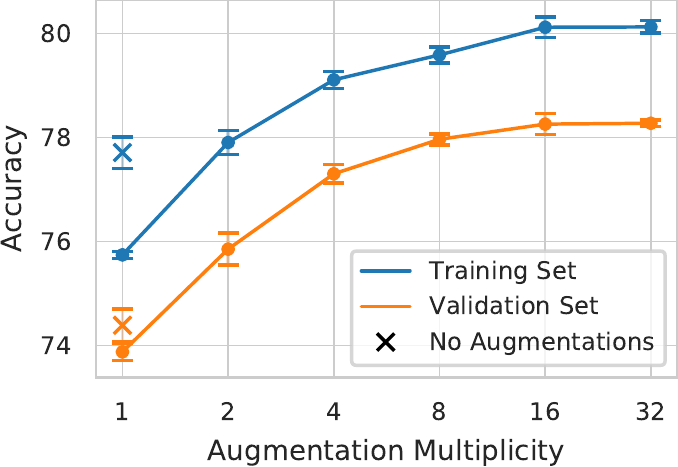}
    \caption{
        Increasing augmentation multiplicities on WRN-16-4 leads to improved training and validation accuracy under $(8, 10^{-5})$-DP. 
        We plot the mean/standard error across 5 independent runs.}
    \label{figure:augmult_ablation}
\end{minipage}
\hspace{0.5cm}
\end{figure}

\paragraph{Large batch sizes.} As observed in several recent papers
\citep{hoory2021learning,Li2021,LuoW0021,anil2021large,DormannFAP21}, increasing the batch size can significantly improve the privacy-utility trade-off of DP-SGD. 
In \Cref{table:method_ablation} we show that increasing the batch size from 256 to 4096 significantly increases both training and validation accuracy. 
Additionally in \Cref{figure:batchsize_ablation}, we evaluate the effect of the batch size on a WRN-16-4 where we sweep the batch size from $2^8 = 256$ to $2^{14} = 16384$. 
We find that performance consistently improves as the batch size rises. 
The effect of increasing the batch size is discussed further in \Cref{sec:cifar_tuning}.

We note that the gap between the training and validation accuracy in \Cref{figure:batchsize_ablation} also increases with the batch size.
As discussed by \citet{yeom2018privacy}, increased overfitting can be an indication of a rise in privacy leakage, and training under DP implies generalization bounds that prevent overfitting when $\varepsilon$ is small \citep{doi:10.1126/science.aaa9375, DBLP:conf/stoc/BassilyNSSSU16}.
However these bounds degrade rapidly with increasing $\varepsilon$, and even the best bounds we are aware of are vacuous in our setting
\citep{DBLP:conf/colt/FeldmanS17,DBLP:conf/innovations/JungLN0SS20}.
In \Cref{app:debugging:lower_bounds} we audit our implementation using membership inference attacks exploiting the generalization gap to assess privacy leakage, and find that the observed gaps are within the theoretical predictions for all values of $\varepsilon$ considered.

\paragraph{Weight standardization.}
Next, we apply Weight Standardization \citep{qiao2019micro} to all convolutional layers, which normalizes the rows of the weight matrix of each convolution over the fan-in of each output unit (See \Cref{app:experimental_details:models} for implementation details). 
Several authors have shown that using weight standardization combined with group normalization can be an effective replacement for batch normalization in non-private training, especially when training with large batch size \citep{qiao2019micro, richemond2020byol, kolesnikov2020big}, while \citet{brock2020characterizing} found Weight Standardization improves the performance of NF-ResNets.
We found that weight standardization also improves the performance of DP-SGD, particularly for large batch sizes. 
In \Cref{table:method_ablation}, we show the improvement in both the training and validation accuracies when using weight standardization with group normalization and a large batch size of 4096 on the WRN-40-4 model.

\paragraph{Augmentation multiplicity.} 
Data augmentation is a crucial component of non-private training, which significantly improves the performance of over-parameterized models on vision tasks \citep{shorten2019survey, zhang2017mixup, cubuk2020randaugment}. 
However, in our experiments with DP-SGD, we found that data augmentation as it is usually implemented in non-private training (using one augmentation per independent example in the mini-batch) reduced both training and validation accuracies (See \Cref{figure:augmult_ablation}). 
We believe this occurs because data augmentation introduces variance into the gradient \citep{Fort2021}, which increases the number of training iterations required to minimize the loss. 
In non-private training, several authors have shown that using multiple augmentations per unique example in a mini-batch can improve performance \citep{hoffer19, Fort2021, touvron2021going}.
We now explore using multiple augmentations per example in the DP-SGD update to see if we can recover the benefits of data augmentation in private training.
The naive approach when using multiple augmentations per example with DP-SGD would be to compute one clipped gradient for each augmented image, however this would lead to a privacy cost scaling with the number of augmentations per training example in each mini-batch.
To circumvent this issue, we average the gradients of different augmentations of the same training example \emph{before} clipping the per-example gradients. 
This approach does not increase the sensitivity of the mini-batch gradient to any single training example, and therefore does not incur any additional privacy cost. 
The resulting update can be written as follows: 
\begin{align}
    w^{(t+1)} = w^{(t)} - \eta_t \left\{ \frac{1}{B} \sum\limits_{i \in \mathcal{B}_t} \frac{1}{C} \texttt{clip}_C \left(\frac{1}{K} \sum\limits_{j\in \mathcal{K}_t^i}  \nabla l_j (w^{(t)}) \right) + \frac{\sigma}{B} \xi \right\} \enspace. \label{eq:augmult}
\end{align}
Here $\mathcal{K}_t^i$ represents the set of $K$ randomly sampled augmentations for the $i^{th}$ training example on the $t^{th}$ iteration. 
We refer to the number of per-example augmentations $K$ as the \emph{augmentation multiplicity}. 
When training with augmentation multiplicity on CIFAR-10, we use random crops and random horizontal flips of the input images, following \citet{DBLP:conf/bmvc/ZagoruykoK16}. 
In \Cref{table:method_ablation}, we show that augmentation multiplicity, as described in \Cref{eq:augmult}, significantly boosts both the training and validation accuracy of the WRN-40-4 model. 
Additionally, we perform a sweep of augmentation multiplicities ranging from 2 to 32 on the WRN-16-4 model at a batch size of 4096 in \Cref{figure:augmult_ablation}, and show that performance consistently improves with higher augmentation multiplicities, significantly exceeding the performance achieved without data augmentation.

\paragraph{Parameter averaging.}
Since the privacy analysis for DP-SGD assumes that all intermediate parameter values obtained during training can be released, parameter averaging techniques do not incur any additional privacy cost. 
We found that using an exponential moving average (EMA) of the parameters consistently improved final accuracy on both the training and validation sets \citep{tan2019efficientnet}, as shown in \Cref{table:method_ablation} (See \Cref{app:experimental_details} for implementation details). 
All results in the remainder of this paper use EMA unless specified otherwise.

\begin{table}[t]
\caption{\label{table:cifar_from_scratch} CIFAR-10 test accuracy of our Wide-ResNet models trained with DP-SGD without additional data.
For our results, we report the median computed over five independent runs as well as the standard deviation. 
}
\begin{center}
\begin{tabular}{lccc}
\toprule [0.15em]
& \multirow{2}{*}{$\varepsilon$} & \multicolumn{2}{c}{Test Accuracy (\%)} \\
\cmidrule{3-4}
& & Median & Std. Dev. \\
\midrule [0.1em]
\multirow{3}{*}{\citet{DBLP:conf/iclr/TramerB21}}
    & 1 & \textbf{60.3} & -- \\
& 2 & \textbf{67.2} & -- \\
& 3 & 69.3 & -- \\ \midrule
\multirow{3}{*}{\citet{DormannFAP21}} & 1.93 & 58.6 & -- \\
& 4.21 & 66.2 & -- \\
&	7.42 & 70.1 & -- \\ \midrule
\citet{klause2022differentially} & 7.5 & 71.7 & -- \\
\midrule
\multirow{6}{*}{Ours (WRN-16-4)} 
     & 1 & 56.8 & {\color{gray} (0.6)} \\
     & 2 & 64.9 & {\color{gray} (0.5)} \\
     & 3 & 69.2 & {\color{gray} (0.3)} \\
     & 4 & 71.9 & {\color{gray} (0.3)} \\
     & 6 & 77.0 & {\color{gray} (0.8)} \\
     & 8 & 79.5 & {\color{gray} (0.7)} \\
\midrule
\multirow{6}{*}{Ours (WRN-40-4)} 
     & 1 & 56.4          & {\color{gray} (0.6)} \\
     & 2 & 65.9 & {\color{gray} (0.5)} \\
     & 3 & \textbf{70.7} & {\color{gray} (0.2)} \\
     & 4 & \textbf{73.5} & {\color{gray} (0.6)} \\
     & 6 & \textbf{78.8} & {\color{gray} (0.4)} \\
     & 8 & \textbf{81.4} & {\color{gray} (0.2)} \\
\bottomrule[0.15em]
\end{tabular}
\end{center}
\end{table}

\subsection{Training on CIFAR-10 Without Additional Data -- Official Evaluation}
\label{sec:cifar_sota}

In \Cref{sec:convergence}, we described a collection of techniques which substantially enhanced the performance of DP-SGD when training Wide-ResNets on CIFAR-10 \citep{DBLP:conf/bmvc/ZagoruykoK16}. 
For those experiments, we split the dataset of 50K examples into a training set of 45K examples and a validation set of 5K examples, and did not evaluate on the held out test set. 
In this section, we take the insights identified above, re-train our models on the full dataset of 50K examples, and evaluate performance on the official test set of 10K examples.
We run experiments on both WRN-16-4 and WRN-40-4, at a range of $\varepsilon$ values. 
To determine the best hyper-parameters, we first train on the reduced training set of 45K examples, tuning the learning rate $\eta$ and the noise parameter $\sigma$ on the validation set. We then expand the training set to 50K examples and re-train without additional tuning. 
We report the median accuracy of 5 independent training runs on the CIFAR-10 test set. 
All experiments in this section apply group normalization, weight standardization and parameter averaging as described above.

We first consider the WRN-16-4 model. Using a batch size of $B = 4096$ and augmentation multiplicity $K = 16$, this model achieves a held-out test accuracy of 78.7\% under $(8, 10^{-5})$-DP.
Next, we consider the deeper WRN-40-4 model, and scale up the batch size $B$ to $2^{14} = 16384$ and the augmentation multiplicity $K$ to $32$. 
With this setting, we achieve our best test accuracy on CIFAR-10 of 81.4\% under $(8, 10^{-5})$-DP. 
This exceeds the previous SOTA on CIFAR-10 for this privacy budget \citep{klause2022differentially} by 9.7\%. 
We also provide the test accuracy across a range of $\varepsilon$ values between 1 and 8. 
For these experiments, we tune the learning rate $\eta$ and the noise parameter $\sigma$ independently for each $\varepsilon$ (on the validation set as described above). 
As shown in \Cref{fig:headline_a} and the more detailed comparison in \Cref{table:cifar_from_scratch}, we improve the current SOTA for all values of $\varepsilon \geq 3$.
We note that for small values of $\varepsilon$, the model used by \citet{DBLP:conf/iclr/TramerB21} performs best, likely due to its ability to get good performance with a low number of updates.
However, the capacity of this model is limited -- only reaching 71\% in non-private training \citep{DBLP:conf/iclr/TramerB21}, and thus the more expressive WRN models outperform it whenever $\varepsilon \geq 3$.
We also note that the WRN-16-4 slightly outperforms the deeper WRN-40-4 model at $\varepsilon = 1$, while the WRN-40-4 performs better at all larger values of $\varepsilon$. 
These observations are consistent with our experiments in \Cref{app:optimal_depth}, which show that the optimal model depth increases with $\varepsilon$.

\subsection{Training on ImageNet Without Additional Data}
\label{sec:imagenet_scratch}

\begin{table}[t]
    \centering
    \caption{Top-1 and top-5 accuracy when training on ImageNet using DP-SGD without additional data.}
    \begin{tabular}{llccc}
        \toprule
         \multirow{2}{*}{Method} & \multirow{2}{*}{Model} & \multirow{2}{*}{$(\varepsilon, \delta)$}  & \multicolumn{2}{c}{Accuracy (\%)} \\
         &&& Top-1 & Top-5 \\
         \midrule
         \citet{Kurakin22} & ResNet-18 & $(13.2, 10^{-6})$ & 6.9 & -- \\
         Ours & NF-ResNet-50 & $(8.0, 8\cdot10^{-7})$  & \bf{32.4} & \bf{55.8}\\
         \bottomrule
    \end{tabular}
    \label{table:imagenet_scratch}
\end{table}

To confirm that the insights above can improve the performance of DP-SGD on large datasets, we now train a classifier with DP from random initialization on the ImageNet dataset \citep{ILSVRC15}. 

We use a Normalizer-Free ResNet-50 (NF-ResNet-50) \citep{brock2020characterizing}, which is a modified version of the standard ResNet-50 model, designed to ensure good signal propagation. 
The model does not contain batch normalization layers and it applies weight standardization to convolutional layers \citep{qiao2019micro}. 
We note that this model has been shown to match the performance of the standard batch normalized ResNet-50 for ImageNet classification in non-private training \citep{brock2020characterizing}. 
We initialize the model parameters using Gaussian random variables as for CIFAR-10 classification \citep{glorot2010understanding}, and we remove all explicit regularization methods such as dropout, stochastic depth, label smoothing and weight decay. 
We reserve 10K images from the official training set of ImageNet as our validation set to tune hyper-parameters, and train on the remaining 1.27 million training examples. 
We use the official validation set of 50K examples as our test set. 
We use a batch size of 16384 and set augmentation multiplicity $K=4$. See \Cref{app:experimental_details} for further experimental details. 
Training on ImageNet is significantly more computationally expensive than on CIFAR-10. 
We therefore only ran a very limited sweep of the learning rate $\eta$ and the noise parameter $\sigma$ for a single random seed.

As shown in \Cref{table:imagenet_scratch}, our NF-ResNet-50 model achieves a top-1 accuracy of 32.4\% and a top-5 accuracy of 55.8$\%$ under $(8, 8\cdot10^{-7})$-DP. 
We note that this is a significant improvement on the previous SOTA for training on ImageNet with DP without using additional data, which achieves a top-1 accuracy of 6.9$\%$ under a larger privacy budget of $(13.2, 10^{-6})$-DP using a ResNet-18 model \citep{Kurakin22}.

\section{High-Accuracy Differentially Private Image Classification with Fine-Tuning}
\label{sec:pre-training}

We now consider the setting where we have access to a large non-sensitive/public dataset on which we can pre-train our model, using standard non-private training. 
We then fine-tune this pre-trained model with DP-SGD on the sensitive/private dataset. 
This approach has been recently used successfully when privately fine-tuning language models \citep{Li2021, Yu2021lm}, and it has also been shown to improve the performance of DP-SGD on both CIFAR-10 \citep{abadi2016deep, DBLP:conf/iclr/TramerB21, YuZ0L21} and ImageNet \citep{Kurakin22}. 
However, the best reported performance of DP-SGD on image classification benchmarks is still significantly lower than non-private training, particularly on large challenging datasets like ImageNet.

Using the techniques described in \Cref{sec:dptraining_protocol}, we now show that we can significantly reduce the performance gap between private and non-private models in the fine-tuning regime. 
In \Cref{sec:cifar_finetune} we provide results when fine-tuning on CIFAR-10 and CIFAR-100 from a model pre-trained on ImageNet. 
In \Cref{sec:imagenet} we provide results when fine-tuning on ImageNet using models pre-trained on either JFT-300M or JFT-4B \citep{sun2017revisiting}, two internal datasets containing 300 million and 4 billion labelled images respectively. 
Finally, to confirm the versatility of our approach, in \Cref{sec:places} we fine-tune from JFT-300M to Places-365 \citep{zhou2017places}, a challenging scene recognition dataset.
For all experiments in this section, we explore both fine-tuning the final layer and fine-tuning all layers simultaneously.
Except on ImageNet, almost all our best results are obtained by fine-tuning all layers of the model simultaneously.

\subsection{Fine-tuning on CIFAR-10 and CIFAR-100}
\label{sec:cifar_finetune}

In this section, we consider differentially-private fine-tuning on the CIFAR-10 and CIFAR-100 datasets, using models that were pre-trained on ImageNet without privacy. 
We consider WRN-28-10 and WRN-40-4 models, which we pre-train on ImageNet using images down-sampled to $32\times32$ \citep{chrabaszcz2017downsampled}. 
For pre-training, we use SGD with momentum for 240K iterations with a constant learning rate and a batch size of 1024. We use a weight decay of $5\cdot 10^{-5}$, and apply data augmentation including random flips and crops of the input images with augmentation multiplicity 1 (See \Cref{app:experimental_details} for further experimental details).

\begin{table}[t]
    \centering
    \caption{CIFAR-10 and CIFAR-100 test accuracies when fine-tuning with DP-SGD a 28-10 Wide-ResNet pre-trained on ImageNet (down-sampled to $32\times32$). 
    We report the median accuracy across 5 runs.}
    \begin{tabular}{lc@{\hskip 15pt}cc@{\hskip 15pt}cc}
        \toprule [0.15em]
         \multirow{3}{*}{Fine-tuning Method}
         &
         \multirow{3}{*}{$\varepsilon$}
         &
         \multicolumn{4}{c}{Test Accuracy (\%)} \\
         \cmidrule{3-6}
         &  & \multicolumn{2}{c}{CIFAR-10} & \multicolumn{2}{c}{CIFAR-100} \\
         \cmidrule(lr){3-4} \cmidrule(lr){5-6}
         & & Median & Std. Dev. & Median & Std. Dev. \\
         \midrule [0.1em]
         \multirow{2}{*}{\cite{YuZ0L21}} & 1 & 94.3 & -- & -- & -- \\
         & 2 & 94.8 & -- & -- & -- \\
         \midrule
         \citet{DBLP:conf/iclr/TramerB21} & 2 & 92.7 & -- & -- & -- \\
         \midrule
         \multirow{4}{*}{Classifier layer} & 1 & 93.1 & {\color{gray} (0.03)} & \bf{70.3} & {\color{gray} (0.09)} \\
         & 2 & 93.6 & {\color{gray} (0.05)} & 73.9 & {\color{gray} (0.32)} \\
         & 4 & 94.0 & {\color{gray} (0.08)} & 76.1 & {\color{gray} (0.31)} \\
         & 8 & 94.2 & {\color{gray} (0.07)} & 77.6 & {\color{gray} (0.10)} \\
         \midrule
         \multirow{4}{*}{All layers} & 1 & \bf{94.8} & {\color{gray} (0.08)} & 67.4 & {\color{gray} (0.17)}\\
         & 2 & \bf{95.4} & {\color{gray} (0.15)} & \bf{74.7} & {\color{gray} (0.15)} \\
         & 4 & \bf{96.1} & {\color{gray} (0.06)} & \bf{79.2} & {\color{gray} (0.24)} \\
         & 8 & \bf{96.6} & {\color{gray} (0.08)} & \bf{81.8} & {\color{gray} (0.07)} \\
         \bottomrule [0.15em]
    \end{tabular}
    \label{table:imagenet_cifar_transfer}
\end{table}

We fine-tune the pre-trained model on both the CIFAR-10 and CIFAR-100 datasets using DP-SGD. 
As in \Cref{sec:dptraining_protocol}, we use a constant learning rate without additional regularization techniques, although we do apply data augmentation with augmentation multiplicity $K=16$. 
We explore two fine-tuning scenarios, one where we fine-tune all layers of the model simultaneously, and a second where we only train the final classifier layer. 
We show our results with the WRN-28-10 model in \Cref{table:imagenet_cifar_transfer}. 
On CIFAR-10, we achieve a new SOTA test accuracy of $94.7\%$ under $(1, 10^{-5})$-DP, exceeding the previous SOTA of $94.3\%$ at the same privacy budget from \citet{YuZ0L21}.
We also provide test accuracies at a range of different privacy budgets for both CIFAR-10 and CIFAR-100. Under $(8, 10^{-5})$-DP, we reach 96.7$\%$ on CIFAR-10 and 81.8$\%$ on CIFAR-100. 
On this task, fine-tuning all the layers of the model simultaneously outperforms fine-tuning only the classifier layer.

We provide additional results for fine-tuning on CIFAR-10 with the WRN-40-4 model in \Cref{app:imagenet_cifar_transfer_40_4}. 
While the 40-4 model also performs well and reaches a test accuracy of $95.6\%$ under $(8, 10^{-5})$-DP, the 28-10 model performs better at all $\varepsilon$ values considered. 
To interpret these results, we note that the WRN-28-10 is a stronger pre-trained model, achieving $62.4\%$ top-1 accuracy on the down-sampled ImageNet validation set, while the WRN-40-4 achieves $57.4\%$ top-1 accuracy. 
This phenomenon was consistent across all our fine-tuning experiments: models which achieved stronger performance on the pre-training task consistently achieved higher test accuracies after private fine-tuning. 
Additionally, in \Cref{app:cifar100_cifar10_transfer}, we provide results for privately fine-tuning on CIFAR-10 using a 40-4 Wide-ResNet model pre-trained non-privately on CIFAR-100, where we again show consistent improvements over previous baselines \citep{LuoW0021} at all values of $\varepsilon$.

\begin{figure}[t]
\begin{minipage}[b]{0.478\linewidth}
    \centering
    \includegraphics{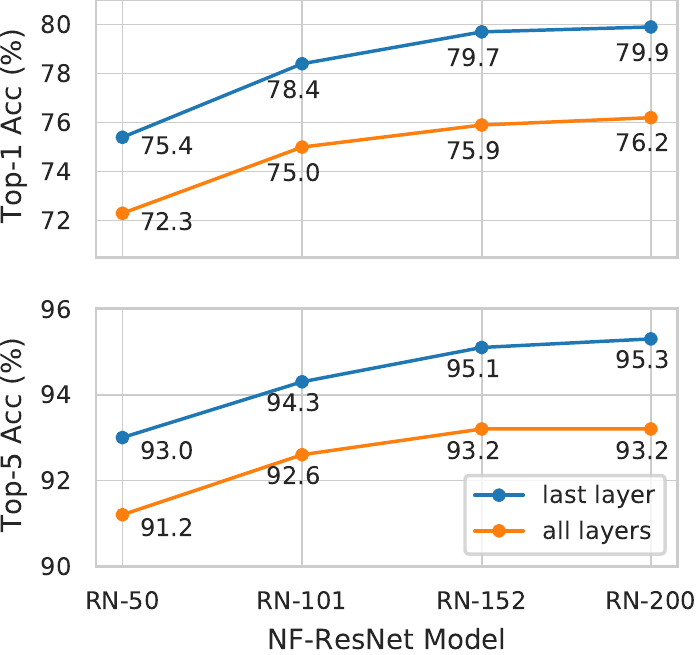}
    \caption{Top-1 and top-5 accuracies of NF-ResNets on ImageNet under $(8, 8\cdot 10^{-7})$-DP, after fine-tuning either the last layer or all layers. To reduce the cost of these experiments we fix $B=1024$ and $\sigma=0.6$.}
    \label{figure:jft-imagenet-depth}
\end{minipage}
\hspace{0.5cm}
\begin{minipage}[b]{0.462\linewidth}
    \centering
    \includegraphics{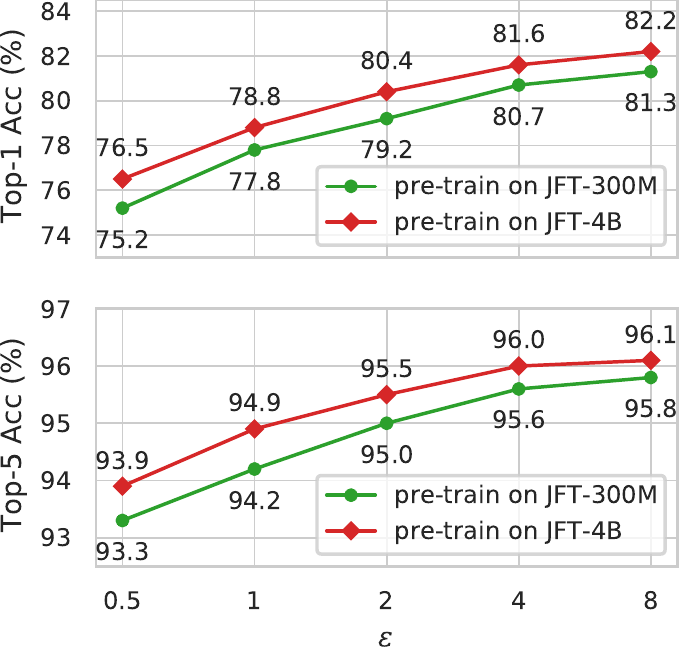}
    \caption{
    Top-1 and top-5 accuracies of NF-ResNet-200 on ImageNet for a range of privacy budgets $\varepsilon$ fine-tuning only the final layer with $B = 2^{18}$. We pre-train either on JFT-300M or on JFT-4B.
    }
    \label{figure:jft-imagenet-sota}
\end{minipage}
\end{figure}

\subsection{Fine-tuning on ImageNet}
\label{sec:imagenet}

In this section, we fine-tune image classifiers with DP-SGD on the ImageNet dataset, using models pre-trained either on JFT-300M or JFT-4B. JFT-300M is an internal dataset of 300 million labelled images from 18k classes, while JFT-4B contains 4 billion images from 30k classes \citep{sun2017revisiting}. 
Pre-training on JFT-300M or JFT-4B has been recently used by several authors to achieve SOTA performance when fine-tuning on ImageNet without privacy \citep{DBLP:conf/icml/BrockDSS21, dosovitskiy2020image, zhai2106scaling}. To preserve privacy guarantees when fine-tuning on ImageNet, we removed all images from both JFT-300M and JFT-4B that were exact or near-duplicates of ImageNet images across common data augmentations \citep{kolesnikov2020big}.\footnote{In an earlier version, we mistakenly used a version of JFT-300M de-duplicated with respect to the validation set of ImageNet but not the training set. After repeating our experiments on a new version of JFT-300M fully de-duplicated with respect to both ImageNet and Places-365, a number of our results improved. We believe this arises because the de-duplication process also removes low-quality images.}

We first use the NF-ResNet architecture \citep{brock2020characterizing} (See \Cref{sec:imagenet_scratch}), which we pre-train on JFT-300M for 10 epochs with the cross entropy loss, using the training procedure described in \citet{DBLP:conf/icml/BrockDSS21}. Interestingly, \citet{DBLP:conf/icml/BrockDSS21} found that NF-ResNets significantly outperformed batch-normalized ResNets when fine-tuning from JFT-300M to ImageNet without privacy. 
We use the same training, validation and test splits for ImageNet as in \Cref{sec:imagenet_scratch}. 
We remove all explicit regularizers from the model such as weight decay or dropout. 
When fine-tuning on ImageNet, we explore both fine-tuning all layers of the network, as well as fine-tuning only the final classifier layer. 
Unless otherwise specified, we fine-tune using DP-SGD with a privacy budget of $(8, 8\cdot 10^{-7})$-DP. 
Because of our computational constraints, we could only afford to run limited sweeps with a single random seed to identify the best hyper-parameters. 
We also do not use augmentation multiplicity to avoid a further increase in the computational cost of training, nor do we use any data augmentation.

We first compare fine-tuning all layers of the model to fine-tuning only the final classifier layer, for a range of NF-ResNet model depths. 
We use a batch size of 1024 and a fixed noise parameter $\sigma = 0.6$, which corresponds to roughly 136 training epochs. 
We tune the learning rate $\eta$ independently for each model. 
Remarkably, we find in \Cref{figure:jft-imagenet-depth} that an NF-ResNet-50 can reach a top-1 accuracy of $75.4\%$ and top-5 accuracy of $93.0\%$ when fine-tuning only the final classifier layer. This far exceeds the previous best reported accuracy for private fine-tuning on ImageNet of $47.8\%$, achieved under a privacy budget of $(10, 10^{-6})$-DP \citep{Kurakin22}. Note that \citet{Kurakin22} used Places-365 \citep{zhou2017places} as the pre-training dataset, which is a significantly smaller dataset compared to JFT-300M.
In \Cref{figure:jft-imagenet-depth}, we also see consistent benefits from increasing the model depth, both when fine-tuning only the last layer and when fine-tuning all layers simultaneously. 
For example, the NF-ResNet-200 model achieves a top-1 accuracy of $79.9\%$ when fine-tuning only the final layer.

We note that, in \Cref{figure:jft-imagenet-depth}, fine-tuning only the final classifier layer consistently performed slightly better than fine-tuning the whole network. 
We therefore only report results when fine-tuning the final classifier layer for our remaining experiments on ImageNet. This also considerably reduces the computational cost of these experiments, since the features can be pre-computed only once and stored on disk -- assuming that no data augmentation is used as is the case for our experiments in this section. This enables us to use a very large batch-size of $B = 2^{18} = 262144$ for the rest of our ImageNet experiments.

We provide the performance of NF-ResNet-200 at a range of privacy budgets $\varepsilon$ in \Cref{figure:jft-imagenet-sota}, where we compare the effect of pre-training either on JFT-300M or on the larger JFT-4B. We pre-train on JFT-300M for 10 epochs and on JFT-4B for 1 epoch.
When privately fine-tuning on ImageNet, we set $T=4000$ (which corresponds to $\sigma=9.1$) when $\varepsilon=8$.
We noticed in \Cref{sec:cifar_sota} that the optimal number of training iterations scaled approximately linearly with $\varepsilon$ when keeping other hyper-parameters constant. 
We use this scaling rule for these experiments to reduce the cost of tuning hyper-parameters. 
For example, at $\varepsilon = 4$ we use half the number of training iterations as our best run with $\varepsilon = 8$. 
The noise parameter $\sigma$ was scaled accordingly, while using the same learning rate.
From \Cref{figure:jft-imagenet-sota}, we see that using a larger and more diverse pre-training dataset consistently improves the accuracy across all values of $\varepsilon$. After pre-training on the larger JFT-4B dataset and scaling up the batch size, the NF-ResNet-200 model achieves a top-1 accuracy of $82.2\%$ and top-5 accuracy of $96.1\%$ under $(8, 8\cdot 10^{-7})$-DP. Note that these experiments required a substantial number of training epochs at $\varepsilon=8$ (roughly 800 epochs at $\sigma=9.1$), significantly larger than is commonly used when fine-tuning in non-private training.

In all our fine-tuning experiments, we found that a stronger pre-trained model, achieved either through using a larger model or through using a larger pre-training dataset, always improved fine-tuning performance.
Indeed, shortly after the first version of our paper was released, \citet{mehta2022large} published a study on differentially private fine-tuning on ImageNet using a large ViT model pre-trained on JFT-4B, in which they achieve 81.1\% top-1 accuracy under (1, $10^{-6}$)-DP and 81.7$\%$ under (4, $10^{-6}$)-DP. Our best results when fine-tuning on ImageNet are achieved with an NFNet-F3 \citep{DBLP:conf/icml/BrockDSS21} pre-trained for 2 epochs on JFT-4B, which we fine-tune for $T=1000$ training iterations using a batch size $B=2^{18}$. We show these NFNet-F3 results in \Cref{table:imagenet_nfnet}. We achieve a top-1 accuracy of $86.7\%$ and top-5 accuracy of $98.0\%$ under $(8, 8\cdot 10^{-7})$-DP.
Most remarkably, this network also achieves a top-1 accuracy of $83.8\%$ under a much tighter privacy budget of $(0.5, 8\cdot 10^{-7})$-DP.

For comparison, fine-tuning the same pre-trained NFNet-F3 without privacy using SGD with Momentum, AGC, and strong additional regularization including Dropout, Weight Decay and Stochastic Depth \citep{DBLP:conf/icml/BrockDSS21} achieves $88.5\%$ top-1 accuracy on ImageNet, which is only $1.8\%$ higher than the accuracy we achieve with DP at $\varepsilon = 8$, and $4.7\%$ higher than the accuracy we obtain at $\varepsilon = 0.5$.
We conclude that our results significantly reduce the gap between private and non-private training on ImageNet when fine-tuning.

\begin{table}[t!]
    \centering
    \caption{ImageNet top-1 accuracy when fine-tuning the last layer of an NFNet-F3 pre-trained on JFT-4B.}
    \begin{tabular}{cccccccc}
    \toprule [0.15em]
    \multirow{2}{*}{Accuracy (\%)} & \multicolumn{6}{c}{$\varepsilon$} & \multirow{2}{*}{Non-private}\\
    \cmidrule(lr){2-7}
    & $0.1$	& $0.5$ & $1.0$	& $2.0$	& $4.0$	& $8.0$	& \\    \midrule [0.1em]
    Top-1 & 77.6 & 83.8 & 84.4 & 85.6 & 86.0 & 86.7 & 88.5  \\
    Top-5 & 93.0 & 96.7 & 96.6 & 97.5 & 97.4 & 98.0 & 98.7 \\
    \bottomrule [0.15em]
    \end{tabular}
    \label{table:imagenet_nfnet}
\end{table}

\begin{table}[t]
    \centering
    \caption{Places-365 accuracy when fine-tuning with and without DP an NF-ResNet-50 pre-trained on JFT-300M.}    
    \begin{tabular}{clcc}
        \toprule [0.15em]
         \multirow{2}{*}{$\varepsilon$} &
         \multirow{2}{*}{Fine-tuning Method}
         & \multicolumn{2}{c}{Accuracy (\%)} \\
         \cmidrule{3-4}
         & & Top-1 & Top-5 \\
         \midrule [0.1em]
         \multirow{2}{*}{8} 
         & Classifier layer & 54.4 & 84.4 \\
         & All layers & \bf{55.1} & \bf{84.6} \\
         \midrule
         \multirow{2}{*}{--} 
         & Classifier layer & 54.5 & 85.3 \\
         & All layers & \bf{57.0} & \bf{87.1} \\
         \bottomrule [0.15em]
    \end{tabular}
    \label{table:places}
\end{table}

\subsection{Fine-tuning on Places-365}
\label{sec:places}

In this section, we provide results for fine-tuning with DP on the Places-365 dataset \citep{zhou2017places}. 
Places-365 is a dataset for scene recognition containing 1.8 million training images from 365 scene categories.
We use the NF-ResNet-50 model \citep{brock2020characterizing} in this section, which is pre-trained on JFT-300M as described in \Cref{sec:imagenet}. 
The current non-private SOTA on this dataset is 60.7$\%$, obtained by a large Vision Transformer \citep{dosovitskiy2020image} after pre-training on a labelled dataset of 3.6 billion public Instagram images \citep{singh2022revisiting}. 
To the best of our knowledge, there is no established SOTA for training on Places-365 with DP.

We fine-tune both with and without privacy, in order to illustrate the scale of the privacy-utility trade-off. 
For both standard and private fine-tuning, no data augmentation is applied and the learning rate is kept constant throughout training. 
The batch size is set to 1024 for non-private training, and to 4096 for private training. 
We do not use any weight decay. 
We tune the learning rate independently for each experiment using a single seed. When using DP-SGD, we set $\delta = 5 \cdot 10^{-7}$. 
We also compare the performance of fine-tuning all layers and fine-tuning only the final classifier layer. 
When training all layers with DP-SGD, we consider a single noise parameter $\sigma = 1$, while when fine-tuning the classifier layer we select the best result from $\sigma \in \{1,2,3\}$.

We present our results in \Cref{table:places}, where we provide the validation accuracy when training with DP-SGD under $(8, 5 \cdot 10^{-7})$-DP, and without privacy using SGD without momentum. 
Training all layers outperforms training the classifier layer for both private and non-private training. 
Our non-private baseline achieves a top-1 validation accuracy of $57.0\%$, which is within a few percent of the current non-private SOTA on this dataset, while our best private model achieves a top-1 accuracy of $55.1\%$, only $1.9\%$ lower than our non-private baseline.

\section{The Interplay between Noise, Batch Size, Compute Budget and Learning Rate}
\label{sec:cifar_tuning}

Carefully tuning the DP-SGD hyper-parameters can significantly boost performance. 
However intuitions from tuning non-private models do not always carry over to private training. For example, in \Cref{figure:batchsize_ablation} we observed an improvement of ${\sim}8\%$ on the validation accuracy after increasing the batch size $B$ and re-tuning the learning rate $\eta$ and noise parameter $\sigma$. 
By contrast, in non-private training under a constant epoch budget, the test accuracy is usually either constant or monotonically decreasing as the batch size rises \citep{smith2017don, smith2020generalization, shallue2018measuring, goyal2017accurate}. 
Below we present some of our findings from an empirical analysis on the impact of different hyper-parameters on the performance of models trained with DP-SGD.

In this section, we only provide experiments on CIFAR-10, using models trained from random initialization without extra data, using the training pipeline described in \Cref{sec:convergence}. 
On this smaller benchmark, we could perform a rigorous empirical evaluation with thorough hyper-parameter sweeps. 
By contrast, on large datasets (e.g., ImageNet), we were unable to tune hyper-parameters to maximize performance, and instead tuned the noise parameter $\sigma$ to ensure the training iteration budget $T$ was feasible within our computational constraints.

\begin{figure}[t]
    \centering
    \subfigure[]{\includegraphics{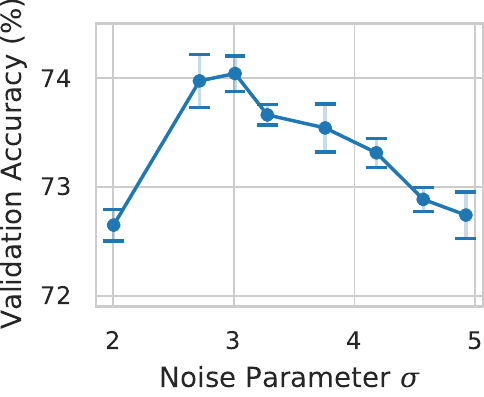}\label{fig:6a}}
     \hfill
    \subfigure[]{\includegraphics{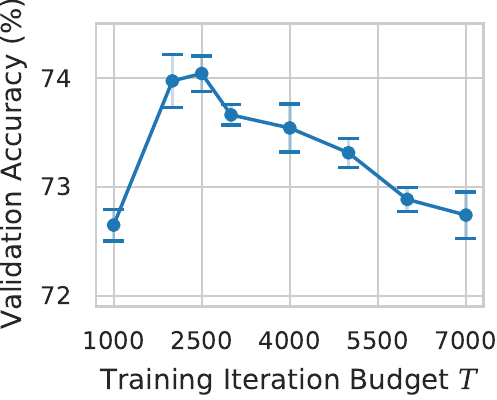}\label{fig:6b}}
    \hfill
    \subfigure[]{\includegraphics{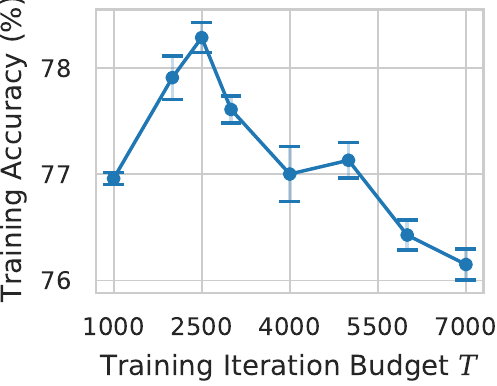}\label{fig:6c}}
    \caption{Training WRN-16-4 on CIFAR-10 at batch size 4096 under $(8, 10^{-5})$-DP. 
    (a) There is an optimal noise parameter, $\sigma_{opt} \approx 3.0$, which maximizes validation accuracy. Note that we tune the learning rate independently for each $\sigma$. 
    (b) Since the training iteration budget is a function of the batch size and the noise parameter, the optimal noise parameter also implies an optimal budget of training iterations, $T_{opt} \approx 2500$. 
    (c) The optimal noise parameter/training iteration budget maximizes training accuracy as well as validation accuracy.
    }
    \label{figure:optimal_compute_train_valid}
\end{figure}

\paragraph{There is an optimal noise parameter at fixed batch size, implying an optimal budget of training iterations.}
Given a privacy budget $(\varepsilon, \delta)$ and a batch size $B$, the noise parameter $\sigma$ determines the maximum budget of training iterations $T$. 
We now investigate whether it is always better to increase the value of $\sigma$ and train longer, or if there is an optimal value of the noise parameter beyond which the performance of DP-SGD degrades.
In \Cref{fig:6a}, we consider a batch size $B=4096$ on the WRN-16-4 under $(8, 10^{-5})$-DP, and we study the validation accuracy achieved across a range of different noise parameters. 
We do not use data augmentation (or augmentation multiplicity) in these experiments, and we re-tune the learning rate on the validation set for each choice of $\sigma$. 
We find that there is an optimal value $\sigma_{opt} \approx 3.0$ that achieves the highest validation accuracy.

The fact there is an optimal value for the noise parameter $\sigma$ implies that there also exists an optimal value for the number of training iterations $T$ (since $\sigma$ and $B$ determine $T$, given a privacy budget), which we illustrate in \Cref{fig:6b}, where the validation accuracy reaches its maximal value for $T_{opt} \approx 2500$.
This observation also holds for the accuracy on the training set, as illustrated in \Cref{fig:6c}. 
By contrast, increasing the number of training iterations almost always achieves higher training accuracy in non-private training \citep{smith2020generalization}.

\begin{figure}[t]
    \centering
    \subfigure[]{\includegraphics{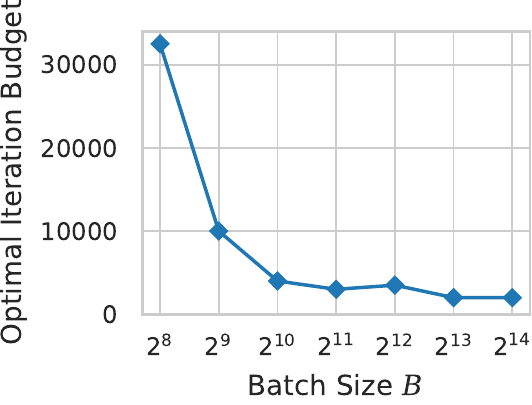}\label{fig:7a}}
    \hfill
    \subfigure[]{\includegraphics{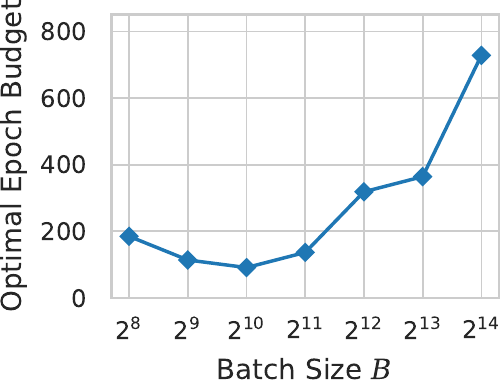}\label{fig:7b}}
    \hfill
    \subfigure[]{\includegraphics{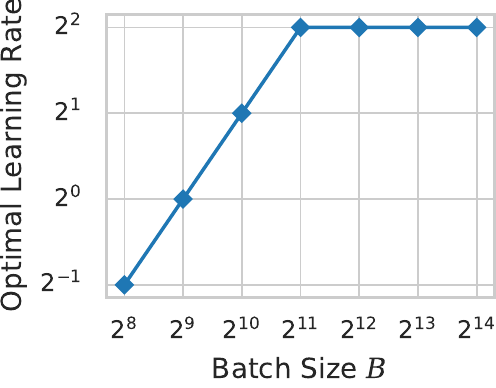}\label{fig:7c}}
    \caption{Training a WRN-16-4 on CIFAR-10 under $(8,10^{-5})$-DP. 
    (a) The optimal training iteration budget falls as the batch size rises for small batch sizes, but is roughly constant for large batch sizes. 
    (b) The optimal epoch budget is roughly constant for small batch sizes, but proportional to batch size for large batch sizes. 
    (c) The optimal learning rate is proportional to batch size for small batch sizes, but constant for large batch sizes. }
    \label{figure:compute_batchsize_scaling}
\end{figure}

\paragraph{The optimal epoch budget increases with the batch size.}
In \Cref{figure:batchsize_ablation}, we showed that increasing the batch size improves the accuracy of DP-SGD (after re-tuning other hyper-parameters). 
Meanwhile, in \Cref{figure:optimal_compute_train_valid} we found that, at a fixed batch size, there is an optimal noise parameter $\sigma_{opt}$ which implies a corresponding optimal budget of training iterations $T_{opt}$.
In \Cref{fig:7a}, we study the same WRN-16-4 model/training pipeline described above, and we plot how the optimal iteration budget $T_{opt}$ (which maximizes validation accuracy after tuning the noise parameter and learning rate) depends on the batch size. 
We find that the optimal iteration budget $T_{opt}$ falls rapidly when the batch size is small, but is roughly constant when the batch size is large. 

However if the number of training iterations is constant, larger batch sizes will perform more training epochs. 
Therefore, in \Cref{fig:7b}, we show how the corresponding optimal epoch budget depends on the batch size. 
We find that the optimal epoch budget is roughly constant when the batch size is small, but increases proportional to the batch size when the batch size is large. 
We conclude that, although large batch sizes achieve superior accuracies with DP-SGD, very large batch sizes also require more compute. 
We recommended that practitioners who are compute constrained choose a moderate batch size in practice (e.g., $B \approx 2048$ for this dataset/architecture). 
Note we provide the train/validation accuracies corresponding to \Cref{figure:compute_batchsize_scaling} in \Cref{figure:batchsize_ablation}.

\paragraph{The optimal learning obeys a linear scaling rule for small batch sizes.} 
In \Cref{fig:7c}, we plot the relationship between the optimal learning rate (after tuning the noise parameter) and the batch size. 
When the batch size is small, the optimal learning rate is proportional to the batch size, while when the batch size is large, the optimal learning rate is constant. 
Similar scaling behaviour has previously been observed in non-private training \citep{mccandlish2018empirical, ma2018power, smith2020generalization}, while linearly scaling the learning rate with the batch size was suggested as a heuristic for private training by \citet{DBLP:conf/iclr/TramerB21}, although as we show this scaling law holds only up to a certain batch size.
Notice that the optimal epoch budget in \Cref{fig:7b} begins to rise at the same batch size where the optimal learning rate ceases to rise in \Cref{fig:7c}.

\section{Discussion}
\label{sec:discussion}

\paragraph{Privacy and fairness considerations.}
We have developed a methodology which achieves high accuracy on academic image classification datasets with DP, however using this methodology for a real-world application on sensitive data would involve several important additional considerations. 
Although in this paper we consider a range of privacy budgets with $\varepsilon$ ranging from below 1 to 8, choosing the privacy budget for a real-world application is non-trivial and would depend on the specific privacy and utility requirements \citep{abowd20censusmodernization}.
In addition, our experiments focus on providing \textit{record-level} privacy at the level of individual images.
In scenarios where a single individual can contribute multiple images to the training data, \textit{user-level} privacy may be preferable \citep{DBLP:conf/iclr/McMahanRT018}.
Furthermore, several authors have shown that differentially private training can potentially reduce the accuracy of models on under-represented subgroups in the dataset \citep{carlini2019distribution, bagdasaryan2019differential, suriyakumar2021chasing, tran2021differentially}, and therefore deploying any application with DP would require careful evaluation, particularly when the dataset is imbalanced. 
Finally, in \Cref{sec:pre-training} we demonstrate that pre-training our networks on large public/non-sensitive datasets significantly enhances the performance achievable by DP-SGD, however this methodology may introduce additional privacy risks if the dataset used for pre-training contains sensitive data \citep{prabhu2020large}.

\paragraph{The computational cost of private training.}
We consistently found that private training requires significantly more compute than non-private training to achieve optimal performance. 
The computational cost of training with DP-SGD can be broken down into two components: the cost of performing a single parameter update given a batch size, and the number of parameter updates that need to be performed for the model to reach a high accuracy.
The cost of a single DP-SGD update is largely dominated by the cost of computing per-example gradients, which is slower than computing the averaged gradient and also requires more memory.
While recent deep learning frameworks like JAX \citep{jax2018github} have significantly reduced these overheads, DP-SGD remains slower than SGD in our experience. 
For example, we observe that computing a parameter update at batch size 1024 for a NF-ResNet-50 on a TPUv3 is about $9\times$ slower with DP-SGD using our implementation than with standard SGD. 
Furthermore, we found that DP-SGD also required more training epochs than non-private training to achieve optimal performance. 
This observation is exacerbated by the use of large batch sizes and/or large augmentation multiplicities, both of which further increase the computational cost of training.
We note that this computational cost is significantly reduced on tasks where fine-tuning only the last layer is sufficient to obtain high accuracy. In our experiments however, we typically found fine-tuning the whole model to be optimal when transferring between datasets other than from JFT-300M/4B to ImageNet.

\paragraph{Correctness.} 
While differentially private algorithms provide strong formal privacy certificates, such guarantees are only realized if the algorithm is correctly implemented, which can be prone to human errors arising from minor variations in the algorithm \citep{DBLP:journals/pvldb/LyuSL17,tramer2022debugging}.
By releasing our implementation we aim to make the code used in this paper verifiable by the DP community.
Additionally, we document key implementation details in \Cref{app:implementation}, and we validate in \Cref{app:debugging:lower_bounds} that the claimed privacy guarantees of our implementation are not contradicted by DP lower bounds obtained using membership inference attacks \citep{tramer2022debugging}. 
Our implementation includes unit-tests, and has undergone two independent internal technical reviews for correctness.

\section*{Conclusion}

We have shown how to unlock high-accuracy image classification with differential privacy by scaling up models, computational budgets, and pre-training data in order to significantly reduce the utility gap between private and non-private image classification on academic datasets. 
For instance, we achieve $83.8\%$ top-1 accuracy on ImageNet under a $(0.5, 8\cdot 10^{-7})$-DP guarantee. 
All of our results are obtained using standard architectures popular in the computer vision community with minimal modification, which should make our methodology easy to re-use in existing pipelines. 
We hope that our insights will significantly reduce the performance penalty that arises when applying differentially private machine learning in real-world applications.

\section*{Acknowledgements}
We would like to thank: Matthias Bauer and Sven Gowal for insightful comments that helped improve the paper's presentation; Robert Stanforth for engineering support, open-sourcing support and code reviews; Rudy Bunel for code reviews; John Aslanides for code quality reviews and open-sourcing support; Alison Reid for support during the open-sourcing process; Andrew Trask and Thomas Steinke for insightful discussions; Zahra Ahmed and Kitty Stacpoole for project management support; Andrew Brock, Taylan Cemgil, Raia Hadsell, Koray Kavukcuoglu, Pushmeet Kohli, Razvan Pascanu and Yee Whye Teh for advice throughout the project.
We would also like to thank Abhradeep Thakurta, Florian Tram{\`e}r and Harsh Mehta for discussions on their related works.

\bibliographystyle{abbrvnat}
\setlength{\bibsep}{5pt} 
\setlength{\bibhang}{0pt}
\bibliography{main}

\pagebreak
\appendix

\newpage
\section{Implementing and Auditing DP-SGD}
\label{app:jax-dp-sgd}

All the experiments reported in this paper use a JAX \citep{jax2018github} implementation of DP-SGD based on JAXline \citep{deepmind2020jax}, a re-usable framework for distributed model training and evaluation.
Our full implementation, including launch scripts for reproducing experiments that do not depend on models pre-trained on JFT-300M, is available at
\begin{center}
    \url{https://github.com/deepmind/jax_privacy}
\end{center}
Besides enabling reproducibility of our results, another important reason for open sourcing our code is to allow the differential privacy community to verify our implementation of DP-SGD.
To help navigate our code base, \Cref{app:implementation} provides an in-depth description of how our code parallelizes the computation of privatized gradients across many devices when using virtual batching and multiple augmentations.
Furthermore, \Cref{app:debugging:lower_bounds} describes how we used the methodology proposed by \citet{tramer2022debugging} to empirically test our implementation against membership inference attacks.

Our implementation leverages the default R{\'e}nyi DP privacy accountant for DP-SGD provided by the TensorFlow Privacy \citep{tensorflowprivacy} implementation.
As a minor technical note, we remark that this accountant assumes that at each iteration mini-batches are sampled with replacement from the entire dataset by including every training example with probability $q$, while in practice we sample mini-batches using a random shuffling scheme,
such that each example is sampled once per training epoch. This discrepancy commonly arises in empirical studies of DP-SGD, as well as in popular open-source implementations \citep{tensorflowprivacy,opacus}.
The main driver behind this discrepancy is the pervasiveness of shuffling-based mini-batching in deep learning frameworks. Obtaining tight numerical accounting methods for shuffled DP mechanisms based on Gaussian perturbations remains an open problem \citep{DBLP:conf/focs/FeldmanMT21,DBLP:journals/corr/abs-2106-00477}.

\subsection{Implementation Details}\label{app:implementation}

\newcommand{\Nacc}{N_{acc}}
\newcommand{\Ndev}{N_{dev}}
\newcommand{\Naug}{K}
\newcommand{\mlocal}{B_{local}}

\begin{algorithm}
\DontPrintSemicolon
\KwIn{Current model parameters $w$, clipping norm $C$, noise multiplier $\sigma$, device id $d$, per-device per-step batch-size $\mlocal$, number of gradient accumulation steps $\Nacc$, number of devices $\Ndev$, number of per-example augmentations $\Naug$,
training examples
$\{(x_{d,s,i}, y_{d,s,i}) : s \in [\Nacc], i \in [\mlocal] \}$,
shared noise samples $\xi_1, \ldots, \xi_{\Nacc} \stackrel{iid}{\sim} \mathcal{N}(0,I)$}
$B \gets \mlocal \cdot \Ndev \cdot \Nacc$\;
$g \gets 0$\;
\For{$s \in \{1,\ldots, \Nacc\}$}{
    \For{$i \in \{1, \ldots, \mlocal\}$}{
        $g_{d,s,i} \gets \frac{1}{C} \texttt{clip}_C\left(\frac{1}{\Naug} \sum_{j=1}^{\Naug} \nabla \mathcal{L}(w, \texttt{augment}(x_{d,s,i}), y_{d,s,i})\right)$\;
    }
    $g_{d,s} \gets \frac{1}{\mlocal} \sum_i g_{d,s,i}$ \tcp*[h]{Average over local mini-batch}\;
    $\hat{g}_{d,s} \gets g_{d,s} + \frac{\sigma \sqrt{\Nacc}}{B} \xi_s$ \tcp*[h]{Add the same (scaled) noise on each device}\;
    $\bar{g}_{s} \gets \frac{1}{\Ndev} \sum_{d'} \hat{g}_{d',s}$ \tcp*[h]{Synchronize average gradient across devices}\;
    $g \gets g + \frac{\bar{g}_{s} - g}{s}$ \tcp*[h]{Numerically stable averaging across accumulation steps}\;
}
\KwRet{$g$} \tcp*[h]{Each device gets the same gradient}

\caption{Private gradient computation across multiple devices with virtual-batching, multiple augmentations, synchronized noise and gradient normalization.}\label{alg:dpsgd-jaxline}
\end{algorithm}

\Cref{alg:dpsgd-jaxline} provides a high-level description of how our DP-SGD implementation computes the privatized gradient used in each model update step. The structure of the code is informed by the way model training pipelines are implemented in JAXline.
The implementation is parallelized across $\Ndev$ devices, where each device runs a copy of \Cref{alg:dpsgd-jaxline}. To extract the maximum possible throughput from the implementation, each device processes training examples in batches of size $\mlocal$, where this parameter is adjusted depending on the memory available in each device and the size of model gradients for the present architecture. To accommodate settings where the desired batch size for a single model update is larger than $\mlocal \cdot \Ndev$, our implementation incorporates gradient accumulation (i.e.\ virtual batching) where $\Nacc$ gradient accumulation steps are performed before each model update, giving a total batch size $B = \mlocal \cdot \Ndev \cdot \Nacc$.

As input to the gradient computation step, each device receives the current model parameters $w$ (which are identical across devices), the desired clipping norm $C$ and noise standard deviation $\sigma$, and their device identifier $d \in \{1, \ldots, \Ndev\}$. In addition, each device $d$ has access to $\Nacc \cdot \mlocal$ training examples $\{(x_{d,s,i}, y_{d,s,i}) : s \in [\Nacc], i \in [\mlocal] \}$, and $\Nacc$ i.i.d.\ samples from a standard Gaussian distribution $\xi_1, \ldots, \xi_{\Nacc}$. Crucially, these samples are \emph{shared} across devices, meaning that noise is independent across gradient accumulation steps but not across devices. This is enforced in our implementation by broadcasting the same pseudo-random number generator key to all devices -- this is preferred over having a different PRNG key per-device because it makes the pipeline more reproducible across training infrastructures with different numbers of devices.  

In the innermost loop of \Cref{alg:dpsgd-jaxline}, client $d$ computes the individual contribution $g_{d,s,i}$ of a single example $x_{d,s,i}$ (where $s$ indexes the accumulation step and $i$ the local batch-size). As discussed in \Cref{sec:convergence}, in our algorithm this contribution is based on the average parameter gradient with respect to $\Naug$ random augmentations of $x_{d,s,i}$. These augmentations are returned by \emph{independent} calls to the $\texttt{augment}$ subroutine. The result of this average is then clipped to a maximal $\ell_2$ norm of $C$ by the $\texttt{clip}_C$ function and then normalized by $C$ (cf.\ \Cref{eq:dpsgd_rescaled}). This results in a contribution $g_{d,s,i}$ to the model update with norm bounded by 1, which is then averaged locally over $i$ on device $d$ to produce $g_{d,s}$. We note that although the loop over $i \in [\mlocal]$ is presented in \Cref{alg:dpsgd-jaxline} as a sequential computation for convenience, in reality our implementation uses hardware parallelism offered by modern accelerators -- this means that as long as the device can process $\mlocal \cdot \Naug$ examples in parallel, the cost of computing $g_{d,s}$ is constant in these parameters.

After computing the average of clipped gradient contributions $g_{d,s}$, each device $d$ adds appropriately calibrated Gaussian noise to obtain $\hat{g}_{d,s}$ -- we emphasize again that in step $s$ each device shares the same random noise sample, so the noise adding process is not independent across devices. At this point devices synchronize their updates by computing the average of $\hat{g}_{d,s}$ over $d \in [\Ndev]$. After this step, each device has the same averaged noisy gradient $\bar{g}_s$. Finally, each device updates their local copy of the accumulated gradient $g$ by using Welford's incremental averaging algorithm for numerical stability \citep{welford1962note}.

That \Cref{alg:dpsgd-jaxline} provides a correct implementation of the privatized gradients required by DP-SGD follows by comparing the following result to \Cref{eq:dpsgd_rescaled}.

\begin{lemma}
Each device participating in \Cref{alg:dpsgd-jaxline} obtains the same noisy gradient given by
\begin{align}
    g = \left\{\frac{1}{B} \sum_{d=1}^{\Ndev} \sum_{s = 1}^{\Nacc} \sum_{i = 1}^{\mlocal} \frac{1}{C} \texttt{clip}_C\left(\frac{1}{\Naug} \sum_{j=1}^{\Naug} \nabla \mathcal{L}(w, \texttt{augment}(x_{d,s,i}), y_{d,s,i})\right)\right\} + \frac{\sigma}{B} \xi \enspace,
    \label{eq:dpsgd-jaxline}
\end{align}
where $\xi \sim \mathcal{N}(0, I)$.
\end{lemma}
\begin{proof}
By construction it is clear that each device gets the same gradient.
Now let $g^{(s)}$ be the value of $g$ (on an arbitrary device) at the end of iteration $s$ of the outermost loop. By induction on $s$ we can show that $g^{(s)} = \frac{1}{s} \sum_{s'=1}^{s} \bar{g}_{s'}$: it is clear that $g^{(1)} = \bar{g}_1$, and
\begin{align*}
g^{(s+1)}
=
g^{(s)} + \frac{\bar{g}_{s+1} - g^{(s)}}{s+1}
=
\frac{s g^{(s)} + \bar{g}_{s+1}}{s+1}
=
\frac{1}{s+1} \sum_{s'=1}^{s+1} \bar{g}_{s'} \enspace,
\end{align*}
where the last identity follows by the inductive hypothesis.
Thus, at the end of the algorithm each device gets $g = g^{(\Nacc)} = \frac{1}{\Nacc} \sum_{s=1}^{\Nacc} \bar{g}_s$. Unrolling the computations done at every accumulation step on every device we get:
\begin{align*}
    g
    &=
    \frac{1}{\Nacc} \sum_{s=1}^{\Nacc} \bar{g}_s
    \\
    &=
    \frac{1}{\Nacc \cdot \Ndev} \sum_{s=1}^{\Nacc} \sum_{d=1}^{\Ndev} \hat{g}_{d,s}
    \\
    &=
    \frac{1}{\Nacc \cdot \Ndev} \sum_{s=1}^{\Nacc} \sum_{d=1}^{\Ndev} g_{d,s} + \frac{\sigma}{\sqrt{\Nacc} \cdot \Ndev \cdot B} \sum_{s=1}^{\Nacc} \sum_{d=1}^{\Ndev} \xi_{s}
    \\
    &=
    \frac{1}{\Nacc \cdot \Ndev \cdot \mlocal} \sum_{s=1}^{\Nacc} \sum_{d=1}^{\Ndev} \sum_{i=1}^{\mlocal} g_{d,s,i}
    +
    \frac{\sigma}{\sqrt{\Nacc} \cdot B} \sum_{s=1}^{\Nacc} \xi_{s}
    \enspace.
\end{align*}
The result now follows by observing that the first term in the sum above equals the first term in \Cref{eq:dpsgd-jaxline} and $\frac{1}{\sqrt{\Nacc}} \sum_{s=1}^{\Nacc} \xi_{s} \sim \mathcal{N}(0,I)$.
\end{proof}

\subsection{Privacy Lower Bounds via Membership Inference Attacks}
\label{app:debugging:lower_bounds}

Inspired by recent work on auditing DP algorithms with privacy attacks \citep{jagielski2020auditing, nasr2021adversary, tramer2022debugging}, we carry out membership inference attacks on models trained with \Cref{alg:dpsgd-jaxline}.
The extent to which these attacks are successful can be used to provide (empirical) \emph{lower bounds} on the privacy guarantees afforded by our implementation, which can then be compared with the nominal upper bound.
Note that this type of test is incomplete; a failure to find a violation of the upper bound does not rule out the possibility that one exists. 
However, along with independent code reviews and unit testing, it provides another signal that our implementation of DP-SGD is correct.

\paragraph{DP lower bounds via hypothesis testing.}
Given two datasets, $D$ and $D'=D\cup\{z\}$, that differ by a single data point $z$, a model trained with a DP algorithm reduces an adversary's ability to infer via a simple hypothesis test whether the model was trained on $D$ or $D'$.
More formally, if an algorithm is $(\varepsilon, \delta)$-DP then it must satisfy $\ln{(\frac{1-\beta-\delta}{\alpha})}\leq\varepsilon$
\citep{DBLP:journals/jmlr/HallRW13}, where $\alpha$ and $\beta$ denote the Type I and Type II errors of the hypothesis testing procedure.
If we can construct datasets $D$ and $D'$ such that this bound is violated, we can be sure the algorithm does not provide $(\varepsilon, \delta)$-DP;
evaluating $\ln{(\frac{1-\beta-\delta}{\alpha})}$ through a membership inference attack will give a valid lower bound of $\varepsilon$. In practice, the test relies on comparing the loss (distribution) of point $z$ under models trained with $D$ and $D'$.

\paragraph{Overview of the attack.}
Note that DP is a worst-case guarantee, and so we are free to design our datasets $D$ and $D'$ in any way we choose. In addition, our algorithm provides the same privacy guarantees for many setting of its hyper-parameters.
Our attack operates in two phases: in phase I, we design a learning problem that we expect will maximize the ability to infer membership, while in phase II, we use the procedure set out by \cite{tramer2022debugging} to run a membership inference attack to find a (statistically valid) lower bound for $\varepsilon$. 
In phase I, we sweep over different hyper-parameter configurations and choices for $z$, to find a choice that is likely to maximize the ratio $\nicefrac{1-\beta-\delta}{\alpha}$ when the model is trained with \Cref{alg:dpsgd-jaxline}.
In phase II, we use the best configuration from phase I and train a large number of models to ensure the lower bound is statistically valid with enough confidence.
For completeness, we repeat the same procedure with two dataset sizes ($|D| \in \{100, 60K\}$) and four settings of the nominal privacy parameters: $\varepsilon\in\{1, 2, 4, 8\}$, where we set $\delta=\nicefrac{1}{|D|}$.

\paragraph{Experimental details.}
Throughout, we will use a simple LeNet classifier \citep{lecun1998gradient} as the trained model and take $D$ from the MNIST dataset.
The hyper-parameters and choices of $z$ we evaluated in phase I are given in \Cref{tab:hp_mia_config}.
Given a choice of $(\varepsilon, \delta)$, we find a choice of hyperparameters that maximizes distinguishability between models trained on $D$ and $D'$ as follows: for each hyperparameter configuration, we train 1K models on $D$ and 1K models on $D'$, where the only randomness originates from the noise added from DP-SGD. 
In particular, all models are initialized with the same parameters, as prior work has shown random initialization weakens privacy attacks \citep{jagielski2020auditing, balle2022reconstructing}. 
We then record the loss on $z$ for each model trained on $D$ and on $D'$, and fit two Gaussians over these histograms.
After this, we record the total variation distance between these two Gaussians \citep{nielsen2018guaranteed}, and choose the hyperparameter configuration that maximizes this distance.
The choice of learning rate, clipping norm, and number of model updates that maximize our ability to infer if $z$ was included in training varied for each $\varepsilon\in\{1, 2, 4, 8\}$ and $|D|\in\{100, 60K\}$. 
However, we found that selecting $z$ to be a blank image was a consistently better choice than using uniform noise or mislabeling an MNIST test set example.

In phase II, we repeat the procedure set out by \cite{tramer2022debugging} for each choice of privacy parameters, dataset size and best hyper-parameters from phase I.
For both $D$ and $D'$ we train 500K models, reserving the first 10K models on each setting as the set from which we find a threshold that maximizes the ratio between $1-\beta$ and $\alpha$, and the remaining models are used to evaluate this chosen threshold and report the $\varepsilon$ lower bound.
To gain statistical confidence in our $\varepsilon$ lower bounds we compute the lower bound of the true positive rate, $1-\beta$, and upper bound of the false positive rate, $\alpha$, over the remaining 980K models using Clopper-Pearson confidence intervals. 
With an appropriate choice of significance level this gives us a probabilistic $\varepsilon$ lower bound with 99.9\% confidence.

\paragraph{Results.}
Our results are given in \Cref{tab:mia_results}; we did not find any violation of our reported $(\varepsilon, \delta)$-DP guarantees. 
Note that our lower bounds when using $|D|=100$ are substantially tighter than when training with $|D|=60K$ -- for example at nominal $\varepsilon>1$ training on a small dataset yields $\varepsilon$ lower bounds which are over $2\times$ stronger.
We also provide membership inference AUC and advantage ($1-\beta-\alpha$) scores for inferring if $z$ was or was not included in training over the 980K models, and give upper bounds to the membership advantage that can be derived in closed form from $\varepsilon$ and $\delta$ (c.f. \cite{humphries2020differentially}).
Throughout our entire experiment we chose to train models without augmentation multiplicity to expedite experimentation, however, we ran an additional experiment for $\varepsilon=1$ and $|D|=100$ where we turned on augmentation multiplicity during training with 16 augmentations per image, and found the $\varepsilon$ lower bound was almost unchanged.

\begin{table}[H]
\centering
\caption{Phase I of the $\varepsilon$ lower bound experiment: Finding the best choice of hyperparameters that distinguish models trained with and without $z$.}
\label{tab:hp_mia_config}
\resizebox{\textwidth}{!}{
\begin{tabular}{lc|c}
\toprule
Hyperparameter & \multicolumn{2}{c}{Values}                                                                          \\
\midrule

$z$       & \multicolumn{2}{c}{Uniform noise: \includegraphics[width=0.5cm]{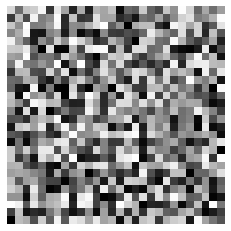}, Blank: \includegraphics[width=0.5cm]{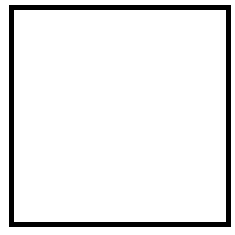}, Label 7 as 8: \includegraphics[width=0.5cm]{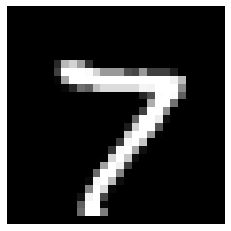},  Label 6 as 7: \includegraphics[width=0.5cm]{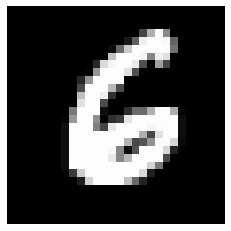}, Label 0 as 1: \includegraphics[width=0.5cm]{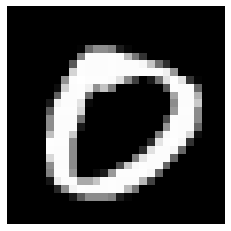}
} \\
Learning rate  & \multicolumn{2}{c}{0.1, 0.5, 1.0}                                                                   \\
Clipping norm  & \multicolumn{2}{c}{0.1, 1.0, 10.0}                                                                  \\
$|D|$ & 60K & 100 \\
Batch size     & 4096 (accumulated over two steps of batch size 2048)              & 64                              \\
Number of updates      & 500, 1K                                                           & 50, 100   \\
\bottomrule
\end{tabular}
}
\end{table}

\begin{table}[H]
\centering
\caption{Phase II of the $\varepsilon$ lower bound experiment: Reporting probabilistic $\varepsilon$ lower bounds with 99.9\% confidence and the membership inference (inferring if $z$ was or was not used in training) AUC and advantage ($1-\beta-\alpha$).}
\label{tab:mia_results}
\resizebox{\textwidth}{!}{
\begin{tabular}{ccccccc}
\toprule
$|D|$ & Nominal $\varepsilon$ & $\varepsilon$ lower bound & Membership AUC & Membership advantage & Membership advantage upper bound \\
\midrule
\multirow{4}{*}{60K}    & 1                                       & 0.279                                                                                                                               & 0.54                                        & 0.06 & 0.46                                              \\
                        & 2                                       & 0.456                                                                                                                             & 0.57                                        & 0.10    & 0.76                                           \\
                        & 4                                       & 1.139                                                                                                                               & 0.62                                        & 0.17     & 0.96                                         \\
                        & 8                                       & 2.153                                                                                                                             & 0.72                                        & 0.31          & 1.00                                    \\

\cmidrule{2-6}
\multirow{4}{*}{100}    & 1                                       & 0.361                                                                                                                             & 0.59                                        & 0.13                     & 0.47                         \\
                        & 2                                       & 0.837                                                                                                                             & 0.66                                        & 0.22     & 0.76                                         \\
                        & 4                                       & 2.461                                                                                                              & 0.79                                        & 0.43        & 0.96                                      \\
                        & 8                                       & 4.327                                                                                                                              & 0.91                                        & 0.67      & 1.00                                        \\
      
\bottomrule
\end{tabular}
}
\end{table}

Note, it is unlikely that taking $D$ to be the MNIST dataset is the optimal choice for maximizing privacy lower bounds; one could imagine taking $D$ to be a pathological dataset with the same pixel value everywhere, and taking $z$ to be an image with different pixel values, would improve distinguishability between models trained on $D$ or $D'$.
However, our aim is to verify that the end-to-end implementation of our DP training procedure is correct, including the data pre-processing part of the pipeline. 
Data pre-processing (including augmentations) may have a null effect on such a pathological dataset, meaning if an error was introduced in this stage of the training pipeline it would not be captured by our membership attack.
Therefore, we chose to use the MNIST dataset to increase the likelihood that any errors from data pre-processing pipeline will be reflected and included in our membership attack.
We leave the design of datasets to elicit the best lower bounds for future work.

\newpage

\section{Additional Experimental Results}\label{app:additional_experiments}

\subsection{Choice of the Clipping Norm under the Modified DP-SGD Update}
\label{app:clipping_rescaling}

In the DP-SGD update as described in \Cref{eq:dpsgd}, the clipping norm $C$ determines the norm of the parameter update, and therefore affects the optimal choice of the learning-rate.
This results in an undesirable dependency between the clipping norm and the learning rate, which can make the performance of DP-SGD appear to be sensitive to the choice of the clipping norm when the learning rate is not carefully tuned \citep{Li2021}.
To decouple the effect of the learning rate and the clipping norm, we instead use a slightly modified version of the DP-SGD update as described in \Cref{eq:dpsgd_rescaled}, where the learning rate $\eta$ absorbs a factor of $C$.
Using this update, performance becomes less sensitive to the choice of clipping norm, as we demonstrate in \Cref{figure:clipping_ablation} where we show the accuracy of the WRN-16-4 over a range of clipping norms and learning rates at batch size 4096 and augmentation multiplicity 16. More importantly, the optimal choice of the learning rate does not depend on the clipping norm under this modified DP-SGD update, as long as the clipping norm is lower or equal to 1.
Therefore, to reduce the cost of hyper-parameter tuning, for all of our experiments in this paper, we fix $C = 1$.

\begin{figure}[H]
    \centering
    \includegraphics{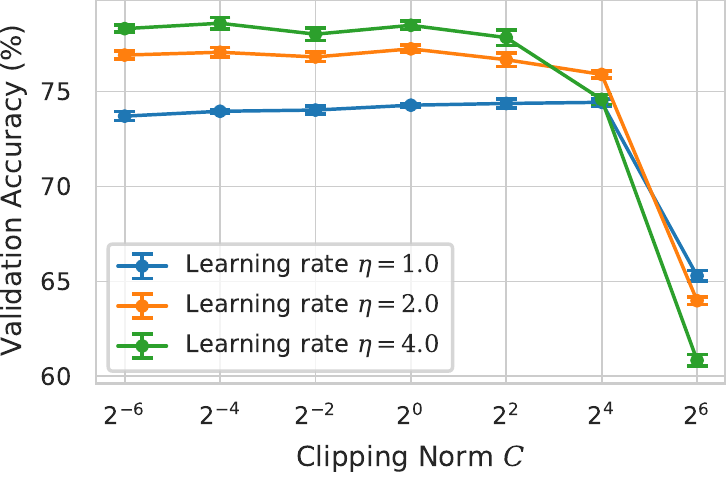}
    \caption{Sensitivity of performance to clipping norm for different learning rates.
    }
    \label{figure:clipping_ablation}
\end{figure}

\subsection{The Optimal Model Depth Depends on the Privacy Budget and Training Setup}
\label{app:optimal_depth}

Scaling the depth of neural networks has been crucial in achieving good performance with non-private training \citep{he2016deep, tan2019efficientnet, DBLP:conf/icml/BrockDSS21}. In this paper, we were interested in exploring whether we could leverage the same benefits of deep models when training with DP-SGD. However, in our initial experiments on WRN models for CIFAR-10 classification without extra data under $(8, 10^{-5})$-DP, we found that both training and validation set performance degraded for networks deeper than the WRN-16-4. 

To understand this, in \Cref{figure:compute_depth_scaling}, we train a range of Wide-ResNets at different depths (from 16 to 100 layers) using a fixed noise parameter $\sigma = 2$ and a batch size of 4096 until exhausting a very large privacy budget of $(64, 10^{-5})$-DP, and track the performance of these models during training. We emphasize that $\varepsilon = 64$ is too large to provide a meaningful notion of privacy, but our goal here is to gain an understanding for the training dynamics.
In our initial experiment, we train without data augmentation. We find that the deeper models start converging slower than the shallower models, likely due to the added noise in DP-SGD. However, when trained for long enough (i.e., for a sufficiently large privacy budget), the deeper models outperform the shallower models. This is illustrated in \Cref{fig:depth_a}, where we show that at low $\varepsilon$ values (which corresponds to early in training for this experiment), the WRN-16-4 outperforms all deeper models, while at high $\varepsilon$ values (i.e., late in training), some of the deeper models start to outperform the shallower models.

Interestingly, we find that using data augmentation with large augmentation multiplicities significantly speeds up convergence of the deeper models. Performing the same experiment as before using augmentation multiplicity $K=16$, we see in \Cref{fig:depth_c} that even early in training at $\varepsilon=8$, the WRN-40-4 significantly outperforms the WRN-16-4. Late in training, at larger values of $\varepsilon$, a much deeper network is optimal compared to when not using augmentation multiplicity. We note this effect of speeding up convergence of deeper models is not simply due to using data augmentation (See \Cref{fig:depth_b} using augmentation multiplicity 1).

We conclude from these results that the optimal network depth depends on the privacy budget, as well as other choices made during training. We notice this in \Cref{table:cifar_from_scratch}, where at a small privacy budget of $(1, 10^{-5})$-DP, the WRN-16-4 slightly outperforms the WRN-40-4. For larger privacy budgets, we find the WRN-40-4 performs better. Additionally, we find the optimal network depth to also depend on whether we train without using any extra data, or if we privately fine-tune a pre-trained network. In the fine-tuning setting, we find larger models to consistently perform better than shallower networks, even at small $\varepsilon$ values, as shown throughout \Cref{sec:pre-training}.

\begin{figure}[H]
    \centering
    \subfigure[$K=0$ (no augmentation)]{\includegraphics{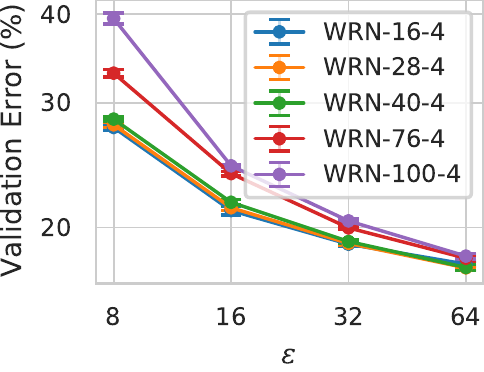}\label{fig:depth_a}}
    \hfill
    \subfigure[ $K=1$]{\includegraphics{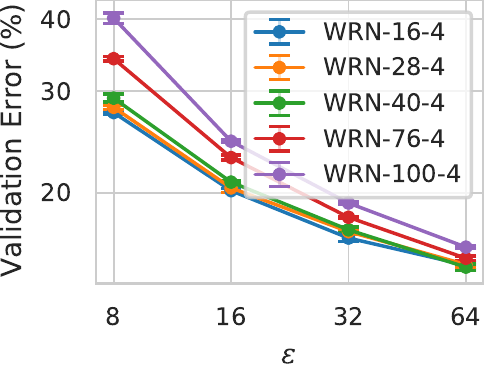}\label{fig:depth_b}}
    \hfill
    \subfigure[ $K=16$]{\includegraphics{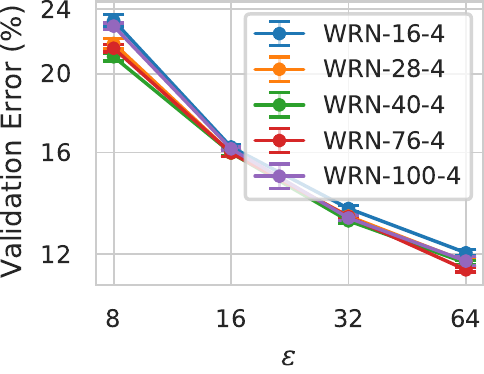}\label{fig:depth_c}}
    \caption{Validation accuracy at different points in training (corresponding to different values of $\varepsilon$) for WRNs at different depths and different augmentation multiplicities $K$, trained with fixed noise parameter $\sigma = 2$.}
    \label{figure:compute_depth_scaling}
\end{figure}

\subsection{Private Fine-Tuning on CIFAR-10 Using a WRN-40-4 Model Pre-trained on ImageNet-32}
\label{app:imagenet_cifar_transfer_40_4}

In \Cref{sec:cifar_finetune}, we provided results for privately fine-tuning a WRN-28-10 on CIFAR-10 and CIFAR-100, which was pre-trained on ImageNet-32. We now provide additional results for this same task using a pre-trained WRN-40-4 model instead, and with the same experimental setup as in \Cref{sec:cifar_finetune}.
We show results in \Cref{table:imagenet_cifar_transfer_40_4}. We fine-tuned all layers of the model simultaneously, which we found to perform better than fine-tuning only the last layer for this experiment. We show that we achieve 95.6\% test accuracy under $(8, 10^{-5})$-DP. As mentioned in \Cref{sec:cifar_finetune}, we find that the WRN-28-10 to outperform the WRN-40-4 at all values of $\varepsilon$ considered, likely due to the WRN-28-10 model being a better pre-trained model.

\begin{table}[H]
    \centering
    \caption{Test accuracy on fine-tuning on CIFAR-10 using a WRN-40-4 pre-trained on ImageNet-32.}
    \begin{tabular}{lcccc}
        \toprule [0.15em]
        Fine-tuning Method & $\varepsilon$ & Test Accuracy (\%) \\
         \midrule [0.1em]
         \multirow{4}{*}{All layers} 
         & 1 & 93.6 \\
         & 2 & 94.3 \\
         & 4 & 95.1 \\
         & 8 & 95.6 \\
         \bottomrule [0.15em]
    \end{tabular}
    \label{table:imagenet_cifar_transfer_40_4}
\end{table}

\subsection{Private Fine-Tuning on CIFAR-10 Using a WRN-40-4 Model Pre-Trained on CIFAR-100}
\label{app:cifar100_cifar10_transfer}

In \Cref{sec:cifar_finetune}, we provided results for privately fine-tuning WRNs on CIFAR-10 and CIFAR-100 using a model pre-trained on ImageNet-32. We now provide additional results, when fine-tuning with DP-SGD on CIFAR-10, using a model pre-trained without privacy on CIFAR-100. 
We compare our results with \citet{LuoW0021}, who provide the current SOTA for this task.
We use the WRN-40-4 model, which we pre-train on CIFAR-100 using SGD with momentum for 100K iterations with a batch size of 128. We use the same fine-tuning procedure on CIFAR-10 as in \Cref{sec:cifar_finetune} (See \Cref{app:experimental_details}), and use batch size 4096 and augmentation multiplicity $K=16$.

We present our results in \Cref{table:cifar_100_cifar_10}. We fine-tuned all layers of the model simultaneously, which we found to perform better than fine-tuning just the last layer for this experiment, similar to \Cref{sec:cifar_finetune}. As before, we outperform previous SOTA results for this task at all privacy budgets considered. In particular we achieve a test accuracy of $80.4\%$ under $(1, 10^{-5})$-DP and a test accuracy of $89.0\%$ under $(8, 10^{-5})$-DP. We note however that the accuracies achieved on CIFAR-10 after pre-training on CIFAR-100 are significantly lower than those obtained after pre-training on ImageNet. This emphasizes the importance of using a strong pre-trained model if we wish to achieve performance competitive with non-private training.

\begin{table}[h]\vspace{-2mm}
    \caption{Test accuracy on CIFAR-10 when privately fine-tuning from the WRN-40-4 pre-trained on CIFAR-100. We report the median test accuracy and the standard deviation across 5 independent runs.}\vspace{-1mm}
	\label{table:cifar_100_cifar_10}
	\centering
	\begin{tabular}{llcc}
		\toprule
		\multirow{2}{*}{$\varepsilon$} & \multirow{2}{*}{Method} & \multicolumn{2}{c}{Test Accuracy (\%)} \\    
		\cmidrule{3-4}
		& & Median & Std. Dev. \\
		\midrule
		\multirow{2}{*}{0.5} & \citet{LuoW0021}  & 73.3 & -- \\
		& Ours      & \textbf{76.7} & {\color{gray} (0.3)} \\
		\midrule
		\multirow{2}{*}{1.0} & \citet{LuoW0021}     & 76.6 & -- \\
		 & Ours     & \textbf{80.4} & {\color{gray} (0.7)} \\
		\midrule
		\multirow{2}{*}{1.5} & \citet{LuoW0021}      & 81.6 & -- \\
		& Ours      & \textbf{83.8} & {\color{gray} (0.3)} \\
		\midrule
		2 & Ours     & \textbf{84.9} & {\color{gray} (0.2)} \\
		4 & Ours & \textbf{87.0} & {\color{gray} (0.2)} \\
		6 & Ours & \textbf{88.4} & {\color{gray} (0.2)} \\
		8 & Ours & \textbf{89.0} & {\color{gray} (0.2)} \\
		\bottomrule
	\end{tabular}
\end{table}

\subsection{Optimal Noise Scale for a Fixed Batch Size: Experiments on the WRN-40-4}
\label{app:optimal_compute_additional}

In \Cref{figure:optimal_compute_train_valid}, we showed that on the WRN-16-4, there is an optimal noise scale, and consequently an optimal compute budget, for a fixed batch size of 4096. To provide additional evidence of this, we show a similar plot in \Cref{figure:optimal_compute_40_4} on a WRN-40-4 trained on CIFAR-10 under $(8,10^{-5})$-DP and using a batch size of 16384 (with all other experimental details the same as in \Cref{figure:optimal_compute_train_valid}). We found this property to hold also at other batch sizes.

\begin{figure}[H]\vspace{-1mm}
    \centering
    \includegraphics{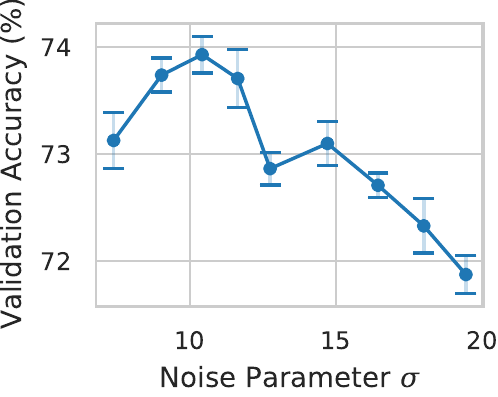} \hfill
    \includegraphics{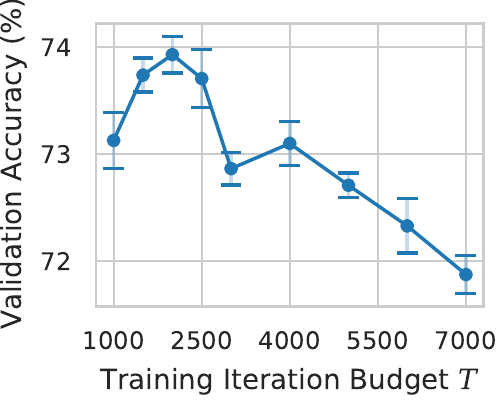} \hfill
    \includegraphics{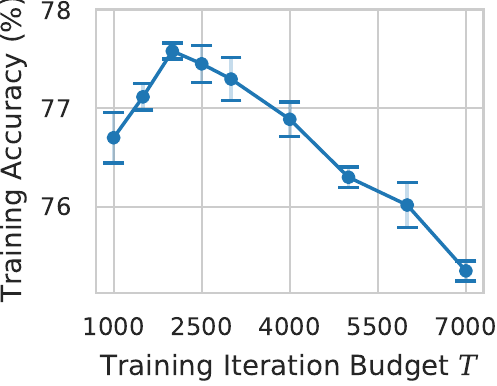}
    \vspace{-1mm}
    \caption{Training a WRN-40-4 on CIFAR-10 at batch size 16384 under $(8,10^{-5})$-DP. 
    }
    \label{figure:optimal_compute_40_4}
\end{figure}

\subsection{Scaling of the Compute Budget with the Batch Size: Experiments on the WRN-40-4}
\label{app:noise_batchsize_scaling}

In \Cref{figure:compute_batchsize_scaling}, we showed how the optimal budget of training iterations and epochs and the optimal learning rate varied  with the batch size on the WRN-16-4 under $(8,10^{-5})$-DP. To provide additional evidence of the main takeaways from \Cref{figure:compute_batchsize_scaling}, we provide additional experiments under the same experimental setup with the WRN-40-4 model in \Cref{figure:compute_batch_scaling_40_4} where we observe qualitatively the same behaviour as in \Cref{figure:compute_batchsize_scaling}. 

\begin{figure}[H] \vspace{-1mm}
    \centering
    \includegraphics{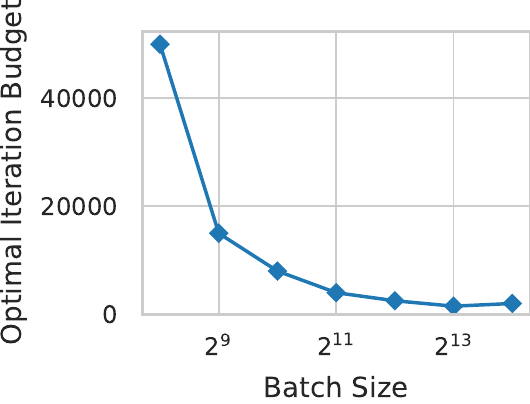} \hfill
    \includegraphics{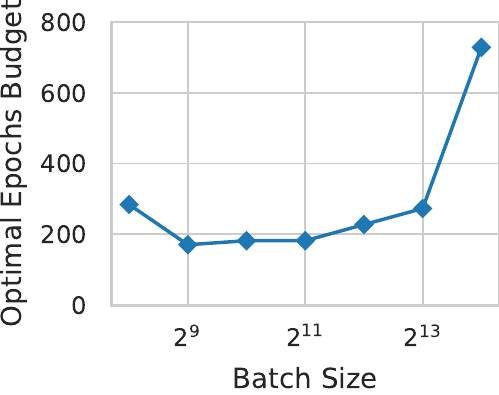} \hfill
    \includegraphics{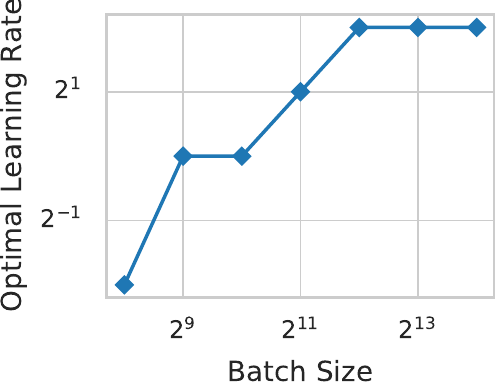}\vspace{-2mm}
    \caption{Optimal compute budgets and optimal learning rates at a range of batch sizes on the WRN-40-4 under $(8,10^{-5})$-DP.}
    \label{figure:compute_batch_scaling_40_4}
\end{figure}

\newpage

\section{Additional Experimental Details}
\label{app:experimental_details}

\subsection{Model Details}
\label{app:experimental_details:models}

\paragraph{Wide-ResNet models.}
We use the Wide-ResNet model family from \citet{DBLP:conf/bmvc/ZagoruykoK16} for all of our experiments on CIFAR-10 and CIFAR-100, and replace all the batch normalization layers with group normalization layers with the number of groups set to 16. We use Gaussian initialization for all weight parameters with variance $(1/\text{fan-in})$ \citep{glorot2010understanding}. We also use Weight Standardization for each convolutional layer as follows:
$$\hat{W}_{i, j} = \dfrac{W_{i, j} - \hat{\mu}_{W_{i, \cdot}}}{\hat{\sigma}_{W_{i,\cdot}}\sqrt{\text{fan-in}}},$$
where $\hat{\mu}_{W_{i, \cdot}}$ and $\hat{\sigma}_{W_{i,\cdot}}$ denote the empirical mean and the standard deviation calculated over the fan-in extent for the convolutional layer $W$. This is similar to the implementation in \citet{brock2020characterizing}. We note for clarity that we employ Weight Standardization throughout training as a differentiable operation. We remove explicit regularizers such as dropout from the model.

\paragraph{NF-ResNet models.}
We use the Normalizer-Free (NF) ResNet family of models from \citet{brock2020characterizing} in most of our experiments on ImageNet and Places-365. When training on ImageNet from random initialization (without using any extra data) as described in \Cref{sec:imagenet_scratch}, we initialize the SkipInit parameter in the NF-ResNets to 1.0, which we found to result in slightly faster convergence when training NF-ResNets with DP-SGD. We remove all explicit regularizers from the model such as dropout and stochastic depth.

\paragraph{NFNet models.} We use the Normalizer-Free F3 Network (NFNet) from \citet{DBLP:conf/icml/BrockDSS21} to achieve our best results when fine-tuning on ImageNet. As before, we remove all explicit regularizers from the model.

\subsection{Dataset Details}
\label{app:experimental_details:datasets}

\paragraph{Overview.}
When using data augmentation, we perform random crops and random horizontal flips at training time. We never use data augmentation at evaluation time. When the augmentation multiplicity is set to 0, it means that no data augmentation is used. Our best hyper-parameters are selected on the validation set, and we report the corresponding performance on the test set for our best results, and on the validation set for all other analysis and ablation experiments. The data splits used for each dataset are reported in \Cref{tab:dataset:split}.

\begin{table}[H]
    \centering
    \caption{Number of samples in each split per dataset.}
    \label{tab:dataset:split}
    \begin{tabular}{lcccc}
        \toprule
         Dataset & Train & Valid & Train + Valid & Test \\
         \midrule
         CIFAR-10 & 45,000& 5,000 & 50,000 & 10,000 \\
         CIFAR-100 & 45,000& 5,000 & 50,000 & 10,000 \\
         \midrule
         ImageNet & 1,271,167 & 10,000 & -- & 50,000 \\
         ImageNet-32 & 1,271,167 & 10,000 & -- & 50,000 \\
         \midrule
         Places-365 & 1,803,460 & -- & -- & 36,500 \\
         \bottomrule
    \end{tabular}
\end{table}

\paragraph{CIFAR.}
When results are reported on the validation set, the model is trained on the $45K$ training images. 
When results are reported on the test set, the model is trained on the full $50K$ training $+$ valid images.
We ensure that the DP accountant uses the corresponding number of examples to calculate the $\varepsilon$ guarantee.
Each image is centered and standardized per channel.
When using random crops, every image is padded with 4 pixels on each edge with \say{mirror} padding.

\paragraph{ImageNet.}
We construct a validation set from the official training set of the ImageNet dataset, so that we can cross-validate hyper-parameters on it. We report our best results on the official validation set. Each image is centered and standardized per channel. For all our experiments training from random initialization without any extra data (as in \Cref{sec:imagenet_scratch}), we use an image size of $224 \times 224$ both during training and evaluation. When fine-tuning on ImageNet from a pre-trained model (as in \Cref{sec:imagenet}), we noticed that using larger image sizes for training and evaluation improved performance, and we use an image size of $320 \times 320$ for all our experiments. When using random crops, we first resize the image to have $1.2$ times the desired height and width, and then take a random crop of the desired size within this large image. 

\paragraph{ImageNet-32.}
We construct a validation set from the official training set similar to the ImageNet dataset.
In the ImageNet-32 dataset, each image from the ImageNet dataset is down-sampled to $32\times 32$. We follow the same image down-sampling procedure as described in \citet{chrabaszcz2017downsampled}. Each image is translated and normalized so as to have values between -1 and 1.
When using random crops, every image is padded with 4 pixels on each edge with \say{mirror} padding. 

\paragraph{Places-365.}
We train on the training set and report results on the test set.
Each image is centered and standardized per channel, and we use an image size of $256 \times 256$ for both training and evaluation.
When using random crops, we first resize the image to have $1.2$ times the desired height and width, and then take a crop of the desired size within this large image.

\subsection{Training Details}

\paragraph{Training with DP-SGD.}
For all our experiments training with DP-SGD (both when training from random initialization or when fine-tuning a pre-trained model), we use a constant learning rate without any learning rate decay. We also do not use explicit regularizers such as weight decay or label smoothing.

We found that using parameter averaging helped reduce the oscillations of the iterates during training and improved accuracy.
While Polyak averaging schemes \citep{polyak1992acceleration} have good theoretical convergence properties in non-private training \citep{rakhlin2011making}, neural network practitioners typically use an exponential moving average (EMA) of the model parameters \citep{tan2019efficientnet, DBLP:conf/icml/BrockDSS21} which is easier to tune.
Following \citet{tan2019efficientnet}, we use EMA with a warm-up schedule where the decay on training iteration $t$ equals $\min(decay\_rate, \frac{1+t}{10+t})$. For our experiments on CIFAR-10 and CIFAR-100, we use a decay rate of 0.9999. For our experiments on ImageNet and Places-365, where we train for a larger number of iterations compared to the CIFAR experiments, we use a decay rate of 0.99999.

\paragraph{Pre-training on ImageNet-32.} 
For non-private pre-training on ImageNet-32 using the WRN-28-10 and the WRN-40-4 (\Cref{sec:cifar_finetune}, we use SGD with a momentum parameter of 0.9. We use a constant learning rate of 0.3 and a weight decay parameter of $5 \cdot 10^{-5}$. We use random crops and random flips of the images with augmentation multiplicity 1. We train for 240K iterations with batch size 1024.

\paragraph{Pre-training on CIFAR-100.}
For non-private pre-training on CIFAR-100 using the WRN-40-4 (\Cref{app:cifar100_cifar10_transfer}), we use SGD with a momentum parameter of 0.9. We use random crops and random flips of the images with augmentation multiplicity 1. We train for 100K iterations with batch size 128. We use an initial learning rate of 0.1, and decay the learning rate by factors of 10 at 50K and 75K iterations. We use a weight decay of $5\cdot 10^{-4}$.

\paragraph{Pre-training on JFT-300M.}
We use the same pre-training procedure on JFT-300M as described in \citet{DBLP:conf/icml/BrockDSS21}. We use SGD with momentum parameter 0.9. We use a learning rate of 0.4, using a warm-up from 0 over 5K iterations and then use cosine annealing to decay the learning rate to 0 through the rest of training. We use random crops and random flips of the images with augmentation multiplicity 1, and train with an image size of $224\times 224$. We train for 10 epochs using a batch size of 1024. We use a weight decay of $10^{-6}$.

\paragraph{Pre-training on JFT-4B.} We use the same procedure for fine-tuning on JFT-4B as described in the previous section on pre-training on JFT-300M, except that we pre-train the NF-ResNet models on JFT-4B  for 1 epoch, instead of 10 epochs. We pre-train the NFNet-F3 models on JFT-4B for 2 epochs.

\subsection{Hyper-parameter Sweeps for the Analyses and Ablations}

For our experiments shown in \Cref{table:method_ablation} and \Cref{figure:batchsize_ablation,figure:optimal_compute_train_valid,figure:compute_batchsize_scaling,figure:augmult_ablation}, we tune the learning rate $\eta$ on a logarithmic grid and the number of training iterations $T$ on a linear grid (and calculate the corresponding noise parameter $\sigma$ for each combination of $\eta$ and $T$ using the privacy accountant)  and made sure the optimal values did not lie on the edge of the hyper-parameter grid. 

\subsection{Optimal Hyper-parameter Values for the Best Results}
\label{app:experimental_details:training}

In the tables below, we now report the hyper-parameters used to achieve our best results for each image classification task considered in this paper.

\begin{table}[H]
    \centering
    \caption{
        Hyper-parameters for training without extra data on CIFAR-10 with a WRN-16-4 (\Cref{table:cifar_from_scratch}).
    }
    \label{tab:experimental_details:training:cifar10-wrn-16-4}
    \begin{tabular}{lcccccc}
        \toprule
         Hyper-parameter & \multicolumn{6}{c}{Value}\\
         \midrule
$\varepsilon$ & 1.0 & 2.0 & 3.0 & 4.0 & 6.0 & 8.0 \\
$\delta$ & $10^{-5}$ & $10^{-5}$ & $10^{-5}$ & $10^{-5}$ & $10^{-5}$ & $10^{-5}$ \\
\midrule
Augmult & 16 & 16 & 16 & 16 & 16 & 16 \\
Batch-size & 4096 & 4096 & 4096 & 4096 & 4096 & 4096 \\
Clipping-norm & 1 & 1 & 1 & 1 & 1 & 1 \\
Learning-rate & 2 & 2 & 2 & 2 & 4 & 4 \\
Noise multiplier $\sigma$ & 10.0 & 6.0 & 5.0 & 4.0 & 3.0 & 3.0 \\
Number Updates & 875 & 1125 & 1593 & 1687 & 1843 & 2468 \\         \bottomrule
    \end{tabular}
\end{table}

\begin{table}[H]
    \centering
    \caption{
        Hyper-parameters for training without extra data on CIFAR-10 with a WRN-40-4 (\Cref{table:cifar_from_scratch}).
    }
    \label{tab:experimental_details:training:cifar10-wrn-40-4}
    \begin{tabular}{lcccccc}
        \toprule
         Hyper-parameter & \multicolumn{6}{c}{Value}\\
         \midrule
$\varepsilon$ & 1.0 & 2.0 & 3.0 & 4.0 & 6.0 & 8.0 \\
$\delta$ & $10^{-5}$ & $10^{-5}$ & $10^{-5}$ & $10^{-5}$ & $10^{-5}$ & $10^{-5}$ \\
\midrule
Augmult & 32 & 32 & 32 & 32 & 32 & 32 \\
Batch-size & 16384 & 16384 & 16384 & 16384 & 16384 & 16384 \\
Clipping-norm & 1 & 1 & 1 & 1 & 1 & 1 \\
Learning-rate & 2 & 2 & 2 & 2 & 4 & 4 \\
Noise multiplier $\sigma$ & 40.0 & 24.0 & 20.0 & 16.0 & 12.0 & 9.4 \\
Number Updates & 906 & 1156 & 1656 & 1765 & 2007 & 2000 \\         \bottomrule
    \end{tabular}
\end{table}

\begin{table}[H]
    \centering
    \caption{
        Hyper-parameters for training without extra data on ImageNet with an NF-ResNet-50 (\Cref{sec:imagenet_scratch}).
    }
    \label{tab:experimental_details:training:imagenet}
    \begin{tabular}{lc}
        \toprule
         Hyper-parameter & \multicolumn{1}{c}{Value}\\
         \midrule
         $\varepsilon$ & 8.0 \\
         $\delta$ & $8\cdot 10^{-7}$\\
         \midrule
        Augmentation multiplicity & 4 \\
        Batch-size & 16384 \\
        Clipping-norm & 1 \\
        Learning-rate & 4 \\
        Noise multiplier $\sigma$ & 2.5 \\
        Number of updates & 71589 \\
         \bottomrule
    \end{tabular}
\end{table}

\begin{table}[H]
    \centering
    \caption{
    Hyper-parameters for ImageNet-32 $\rightarrow$ CIFAR-10, fine-tuning all layers of WRN-28-10 (\Cref{table:imagenet_cifar_transfer}).
    }
    \label{tab:experimental_details:fine-tuning:imagenet-cifar10-full}
    \begin{tabular}{lcccc}
        \toprule
         Hyper-parameter & \multicolumn{4}{c}{Value}\\
          \midrule
          $\varepsilon$ & 1.0 & 2.0 & 4.0 & 8.0 \\
          $\delta$ & $10^{-5}$ & $10^{-5}$ & $10^{-5}$ & $10^{-5}$ \\
          \midrule
        Augmentation multiplicity & 16 & 16 & 16 & 16 \\
        Batch-size & 16384 & 16384 & 16384 & 16384 \\
        Clipping-norm & 1 & 1 & 1 & 1 \\
        Learning-rate & 1 & 1 & 1 & 1 \\
        Noise multiplier $\sigma$ & 21.1 & 15.8 & 12.0 & 9.4 \\
        Number of updates & 250 & 500 & 1000 & 2000 \\
         \bottomrule
    \end{tabular}
\end{table}

\begin{table}[H]
    \centering
    \caption{
    Hyper-parameters for ImageNet-32 $\rightarrow$ CIFAR-10, fine-tuning the last layer of WRN-28-10 (\Cref{table:imagenet_cifar_transfer}).
    }
    \label{tab:experimental_details:fine-tuning:imagenet-cifar10-last}
    \begin{tabular}{lcccc}
        \toprule
         Hyper-parameter & \multicolumn{4}{c}{Value}\\
          \midrule
          $\varepsilon$ & 1.0 & 2.0 & 4.0 & 8.0 \\
          $\delta$ & $10^{-5}$ & $10^{-5}$ & $10^{-5}$ & $10^{-5}$ \\
          \midrule
        Augmentation multiplicity & 16 & 16 & 16 & 16 \\
        Batch-size & 16384 & 16384 & 16384 & 16384 \\
        Clipping-norm & 1 & 1 & 1 & 1 \\
        Learning-rate & 4 & 4 & 4 & 4 \\
        Noise multiplier $\sigma$ & 21.1 & 15.8 & 12.0 & 9.4 \\
        Number of updates & 250 & 500 & 1000 & 2000 \\
         \bottomrule
    \end{tabular}
\end{table}

\begin{table}[H]
    \centering
    \caption{
    Hyper-parameters for ImageNet-32 $\rightarrow$ CIFAR-100, fine-tuning all layers of WRN-28-10 (\Cref{table:imagenet_cifar_transfer}).
    }
    \label{tab:experimental_details:fine-tuning:imagenet-cifar100-full}
    \begin{tabular}{lcccc}
        \toprule
         Hyper-parameter & \multicolumn{4}{c}{Value}\\
          \midrule
          $\varepsilon$ & 1.0 & 2.0 & 4.0 & 8.0 \\
        $\delta$ & $10^{-5}$ & $10^{-5}$ & $10^{-5}$ & $10^{-5}$ \\
        \midrule
         Augmentation multiplicity & 16 & 16 & 16 & 16 \\
        Batch-size & 16384 & 16384 & 16384 & 16384 \\
        Clipping-norm & 1 & 1 & 1 & 1 \\
        Learning-rate & 1 & 1 & 1 & 1 \\
        Noise multiplier $\sigma$ & 21.1 & 15.8 & 12.0 & 9.4 \\
        Number of updates & 250 & 500 & 1000 & 2000 \\
         \bottomrule
    \end{tabular}
\end{table}

\begin{table}[H]
    \centering
    \caption{
        Hyper-parameters for ImageNet-32 $\rightarrow$ CIFAR-100, fine-tuning the last layer of WRN-28-10 (\Cref{table:imagenet_cifar_transfer}).
        }
    \label{tab:experimental_details:fine-tuning:imagenet-cifar100-last}
    \begin{tabular}{lcccc}
        \toprule
         Hyper-parameter & \multicolumn{4}{c}{Value}\\
          \midrule
          $\varepsilon$ & 1.0 & 2.0 & 4.0 & 8.0 \\
        $\delta$ & $10^{-5}$ & $10^{-5}$ & $10^{-5}$ & $10^{-5}$ \\
        \midrule
         Augmentation multiplicity & 16 & 16 & 16 & 16 \\
        Batch-size & 16384 & 16384 & 16384 & 16384 \\
        Clipping-norm & 1 & 1 & 1 & 1 \\
        Learning-rate & 4 & 4 & 4 & 4 \\
        Noise multiplier $\sigma$ & 21.1 & 15.8 & 12.0 & 9.4 \\
        Number of updates & 250 & 500 & 1000 & 2000 \\
         \bottomrule
    \end{tabular}
\end{table}

\begin{table}[H]
    \centering
    \caption{
        Hyper-parameters for CIFAR-100 $\rightarrow$ CIFAR-10, fine-tuning all layers of WRN-40-4 (\Cref{table:cifar_100_cifar_10}).
    }
    \label{tab:experimental_details:fine-tuning:cifar100-cifar10}
    \begin{tabular}{lccccccc}
        \toprule
         Hyper-parameter & \multicolumn{7}{c}{Value}\\
         \midrule
         $\varepsilon$ & 0.5 & 1.0 & 1.5 & 2.0 & 4.0 & 6.0 & 8.0 \\
          $\delta$ & $10^{-5}$ & $10^{-5}$ & $10^{-5}$ & $10^{-5}$ & $10^{-5}$ & $10^{-5}$ & $10^{-5}$ \\
         \midrule
         Augmentation multiplicity & 16 & 16 & 16 & 16 & 16 & 16 & 16 \\
        Batch-size & 1024 & 4096 & 4096 & 4096 & 4096 & 4096 & 4096 \\
        Clipping-norm & 1 & 1 & 1 & 1 & 1 & 1 & 1 \\
        Learning-rate & 0.125 & 1 & 0.3 & 0.3 & 0.3 & 0.1 & 0.1 \\
        Noise multiplier $\sigma$ & 5.0 & 5.0 & 6.0 & 6.0 & 5.0 & 6.0 & 5.0 \\
        Number of updates & 781 & 164 & 531 & 906 & 2171 & 6421 & 7226 \\
         \bottomrule
    \end{tabular}
\end{table}

\begin{table}[H]
    \centering
    \caption{
    Hyper-parameters for JFT-300M/4B $\rightarrow$ ImageNet: fine-tuning last layer of NF-ResNet-200 (\Cref{figure:jft-imagenet-sota}).
    }
    \label{tab:experimental_details:fine-tuning:jft-imagenet-best}
    \begin{tabular}{lccccc}
        \toprule
         Hyper-parameter & \multicolumn{5}{c}{Value}\\
          \midrule
$\varepsilon$ & 0.5 & 1.0 & 2.0 & 4.0 & 8.0 \\
$\delta$ & $8\cdot10^{-7}$ & $8\cdot10^{-7}$ & $8\cdot10^{-7}$ & $8\cdot10^{-7}$ & $8\cdot10^{-7}$ \\
\midrule
Augmult & 0 & 0 & 0 & 0 & 0 \\
Batch-size & 262144 & 262144 & 262144 & 262144 & 262144 \\
Clipping-norm & 1 & 1 & 1 & 1 & 1 \\
Learning-rate & 25.6 & 25.6 & 25.6 & 25.6 & 25.6 \\
Noise multiplier $\sigma$ & 28.7 & 21.2 & 15.7 & 11.8 & 9.1 \\
Number Updates & 250 & 500 & 1000 & 2000 & 4000 \\
         \bottomrule
    \end{tabular}
\end{table}

\begin{table}[H]
    \centering
    \caption{
    Hyper-parameters for JFT-4B $\rightarrow$ ImageNet: fine-tuning last layer of NFNet-F3 (\Cref{table:imagenet_nfnet}).
    }
    \label{tab:experimental_details:fine-tuning:jft-imagenet-nfnet}
    \begin{tabular}{lcccccc}
        \toprule
         Hyper-parameter & \multicolumn{6}{c}{Value}\\
          \midrule
$\varepsilon$ & 0.1 & 0.5 & 1.0 & 2.0 & 4.0 & 8.0 \\
$\delta$ & $8\cdot10^{-7}$ & $8\cdot10^{-7}$ & $8\cdot10^{-7}$ & $8\cdot10^{-7}$ & $8\cdot10^{-7}$ & $8\cdot10^{-7}$ \\
\midrule
Augmult & 0 & 0 & 0 & 0 & 0 & 0 \\
Batch-size & 262144 & 262144 & 262144 & 262144 & 262144 & 262144 \\
Clipping-norm & 1 & 1 & 1 & 1 & 1 & 1 \\
Learning-rate & 30 & 10 & 30 & 30 & 100 & 100 \\
Noise multiplier $\sigma$ & 82.6 & 57.2 & 29.9 & 15.8 & 8.4 & 4.6 \\
Number Updates & 100 & 1000 & 1000 & 1000 & 1000 & 1000 \\
         \bottomrule
    \end{tabular}
\end{table}

\begin{table}[H]
    \centering
    \caption{
    Hyper-parameters for fine-tuning JFT-300M $\rightarrow$ Places-365 (\Cref{table:places}).
    }
    \label{tab:experimental_details:fine-tuning:places}
    \begin{tabular}{lcccc}
        \toprule
         Hyper-parameter & \multicolumn{4}{c}{Value}\\
          \midrule
          Fine-tuning & Entire Model & Last Layer & Entire Model & Last Layer \\
          \midrule
          $\varepsilon$ & 8.0 & 8.0 & $\infty$ & $\infty$ \\
          $\delta$ & $5\cdot10^{-7}$ & $5\cdot10^{-7}$ & 1 & 1 \\
        \midrule
          Augmentation multiplicity & 0 & 0 & 0 & 0 \\
        Batch-size & 4096 & 4096 & 1024 & 1024 \\
        Clipping-norm & 1 & 1 & $\infty$ & $\infty$ \\
        Learning-rate & 0.1 & 0.1 & 0.3 & 0.1 \\
        Noise multiplier $\sigma$ & 1.0 & 2.0 & 0.0 & 0.0 \\
        Number of updates & 223939 & 1374116 & 8250 & 46500 \\
         \bottomrule
    \end{tabular}
\end{table}
 
\end{document}